\documentclass[twocolumn]{svjour3} 

\usepackage{array}
\usepackage{cite}
\usepackage{url}
\usepackage[hidelinks]{hyperref}

\usepackage[pdftex]{graphicx}
\usepackage{subcaption}
\usepackage{float}

\usepackage{amsfonts}
\usepackage{amsmath}
\usepackage{array}
\usepackage{blkarray}
\usepackage{framed}

\def\R{{\mathbb R}}
\def\N{{\mathbb N}}

\def\Diff{{\operatorname{Diff}}}

\def\Imm{\operatorname{Imm}}
\def\Shape{\mathcal{S}}
\def\D{\mathcal{D}}
\def\I{\mathcal{I}}

\def\argmin{{\operatorname{argmin}}}
\def\dist{{\operatorname{dist}}}

\def\Imm{{\mathcal{I}}}
\def\Sol{{\mathfrak{M}}}
\newcommand{\vol}{\operatorname{vol}}

\usepackage[normalem]{ulem}

\journalname{International Journal of Computer Vision}
\def\makeheadbox{\relax}

\usepackage[ruled]{algorithm}
\usepackage[noend]{algpseudocode}

\usepackage{todonotes}

\begin{document}

\title{Elastic shape analysis of surfaces with second-order Sobolev metrics: a comprehensive numerical framework
\thanks{M. Bauer and E. Hartman  were supported by NSF grants DMS-1912037 and DMS-1953244. M.~Bauer was in addition supported by FWF grant FWF-P 35813-N. Y. Sukurdeep and N. Charon were supported by NSF grants DMS-1945224 and DMS-1953267.}
}

\author{Emmanuel Hartman \and
        Yashil Sukurdeep \and
        Eric Klassen \and
        Nicolas Charon  \and
        Martin Bauer
}

\institute{E. Hartman \at
              Florida State University \\
              Department of Mathematics\\
              \email{ehartman@math.fsu.edu}    
           \and
           Y. Sukurdeep \at
              Johns Hopkins University \\
              Department of Applied Mathematics and Statistics\\
              \email{yashil.sukurdeep@jhu.edu} 
           \and
           E. Klassen \at
              Florida State University \\
              Department of Mathematics\\
              \email{klassen@math.fsu.edu}   
           \and
           N. Charon \at
              Johns Hopkins University \\
              Department of Applied Mathematics and Statistics\\
              \email{ncharon1@jhu.edu}
           \and
           M. Bauer \at
              Florida State University \& University of Vienna\\
              Department of Mathematics\\
              \email{bauer@math.fsu.edu}  
}

\date{}

\maketitle

\begin{abstract}
This paper introduces a set of numerical me-thods for Riemannian shape analysis of 3D surfaces within the setting of invariant (elastic) second-order Sobolev metrics. More specifically, we address the computation of geodesics and geodesic distances between parametrized or unparametrized immersed surfaces represented as 3D meshes. Building on this, we develop tools for the statistical shape analysis of sets of surfaces, including methods for estimating Karcher means and performing tangent PCA on shape populations, and for computing parallel transport along paths of surfaces. Our proposed approach fundamentally relies on a relaxed variational formulation for the geodesic matching problem via the use of varifold fidelity terms, which enable us to enforce reparametrization independence when computing geodesics between unparametrized surfaces, while also yielding versatile algorithms that allow us to compare surfaces with varying sampling or mesh structures. Importantly, we demonstrate how our relaxed variational framework can be extended to tackle partially observed data. The different benefits of our numerical pipeline are illustrated over various examples, synthetic and real. 
\end{abstract}

\section{Introduction}
\label{sec:introduction}

\subsection{Motivation}
\label{ssec:motivation}
Over the past decades, advances in imaging techniques and devices have led to significant growth in the quantity and quality of ``shape data" in several fields, such as biomedical imaging, neuroscience, and medicine. By ``shape data", we mean objects whose predominantly interesting features are of geometric and topological nature; examples of which include functions, curves, surfaces or probability densities. Naturally, this prompted the emergence of new mathematical and algorithmic approaches for the analysis of such objects, which led to the development of the growing fields of geometric shape analysis and topological data analysis, see e.g~\cite{younes2010shapes, srivastava2016functional,Kendall1999,edelsbrunner2022computational,carlsson2014topological,bronstein2021geometric,bronstein2008numerical}. 

In this paper, we will focus on $3$D surface data, which is becoming increasingly prominent in several areas due to the emergence of high accuracy $3$D scanning devices. The domain of geometric shape analysis has produced several mathematical frameworks and numerical algorithms for the comparison and statistical analysis of $3$D surfaces that have proven to be useful in numerous applications~\cite{younes2010shapes,srivastava2016functional,jermyn2017elastic}. In the context of shape analysis of surfaces, we distinguish between two fundamentally different scenarios: the analysis of surfaces with known point correspondences (parametrized surfaces), and that of surfaces where the point correspondences are unknown (unparametrized surfaces). In the discrete case, i.e. for simplicial meshes, working with parametrized surfaces thus involves having known one-to-one correspondences between the surfaces' vertices, which implies in particular that their mesh structures are required to be consistent. There exists a plethora of different numerical frameworks for shape analysis of parametrized surfaces, see e.g.~\cite{jermyn2017elastic,kilian2007geometric,rumpf2015bvariational,iglesias2018shape,pierson2022riemannian} and the references therein. Nevertheless, one rarely ever encounters $3$D surface data with consistent mesh structures in practical applications, and thus methods designed for shape analysis of parametrized surfaces are severely limited when used for applications with real data. This motivates the need for registering unparametrized surfaces, i.e., finding the unknown point correspondences between them, as well as the development of tools to compare unparametrized surfaces, i.e. to quantify similarity between them, and for the statistical analysis of unparametrized surfaces.

While algorithms that are designed for the comparison and statistical analysis of unparametrized surfaces usually lead to a counterpart for dealing with parametrized surfaces, the converse is unfortunately far from being true. In the past, this difficulty has been approached by first registering the data in a pre-processing step under a certain metric or objective function, before subsequently comparing and perform statistical analysis of the data independently of this registration metric~\cite{audette2000algorithmic}.
This practice is, however, being increasingly questioned as it is easy to construct examples where it leads to a severe loss in the data structure, see e.g.~\cite{srivastava2016functional} and the references therein. This motivates the need for a more comprehensive solution, where the registration, comparison and statistical analysis are performed jointly. In the field of geometric shape analysis, this is achieved by viewing the space of parametrized surfaces as an infinite-dimensional manifold and equipping it with a Riemannian metric that is invariant under the action of the reparametrization (registration) group. This invariance implies that the Riemannian metric descends to a metric on the quotient space of unparametrized surfaces, which consequently allows us to perform the registration, surface comparison and ensuing statistical analysis in a unified framework.

\subsection{Related work in Riemannian shape analysis}
\label{ssec:related_work}
The Riemannian approach to shape analysis has several benefits. First, a Riemannian metric models a very natural notion of similarity: a Riemannian metric measures the cost of deformations, and can thus be used to define the distance (similarity) between two surfaces as the cost of the cheapest deformation that transforms one surface (the source) onto the other (the target). Furthermore, a Riemannian framework not only leads to a notion of similarity between pairs of surfaces, but also allows one to compute optimal point correspondences and optimal deformations (called geodesics) between the (aligned) surfaces, cf. Figures~\ref{fig:main} and \ref{fig:point_correspondences} for various examples that are computed with the framework of the present paper. Finally, the Riemannian approach directly allows one to apply the methods of geometric statistics~\cite{pennec2006intrinsic,pennec2019riemannian} to develop a comprehensive statistical framework for shape analysis, cf. the  algorithms developed in Section~\ref{sec:statistical_shape_analysis_surfaces}.

Riemannian metrics on spaces of surfaces come in two flavors: intrinsic metrics, which are defined directly on the surface, and extrinsic metrics, which are inherited from right invariant metrics on the diffeomorphism group of $\mathbb R^3$. Intuitively the first approach corresponds to deforming the surface only, while the second approach applies a deformation of the whole ambient space in which the surface is embedded. The latter, which is inherited from the principles of Grenander's pattern theory~\cite{grenander1996elements}, has notably led to the celebrated LDDMM framework, for which powerful numerical toolboxes have been developed~\cite{beg2005computing,charlier2021kernel}. 

The present paper operates in the intrinsic setup. More specifically, we deal with the class of reparametrization invariant Sobolev metrics on spaces of surfaces. This class of metrics was first introduced in the context of spaces of curves by Michor and Mumford~\cite{michor2007overview}, and Mennucci and Yezzi~\cite{mennucci2008properties}. While these two articles focused on the theoretical properties of these Riemannian metrics, several ensuing numerical frameworks have been developed, see e.g.~\cite{srivastava2011shape,bauer2018relaxed,nardi2016geodesics} and the references therein. Subsequently, this framework has been generalized to the space of surfaces by the last author and collaborators~\cite{bauer2011sobolev,bauer2020fractional}. While several of the theoretical results for these metrics on the space of curves have been generalized for the space of surfaces (e.g., local well-posedness of the geodesic equation, non-vanishing geodesic distance), a comprehensive numerical framework is largely missing. 

The most popular numerical approach for shape analysis of surfaces is based on the Square Root Normal Field (SRNF) framework proposed in~\cite{jermyn2012elastic}. This framework defines the SRNF transformation $\phi$, which is a mapping from the space of surfaces $\Imm$ that takes values in $L^2(S^2,\R^3)$, where $S^2$ is the unit sphere. This mapping can then be used to define a (pseudo) distance function on $\Imm$ via the pullback of the $L^2$ distance. This framework is related to intrinsic Riemannian metrics on surfaces as the resulting (pseudo) distance function is a first-order approximation of the geodesic distance of a particular (degenerate) Sobolev metric of order one~\cite{jermyn2012elastic}. The simplicity of the computation of this pseudo-distance has led to several implementations~\cite{laga2017numerical,bauer2021numerical}, which have been shown to be effective in applications,  see e.g.~\cite{kurtek2014statistical,joshi2016surface,matuk2020biomedical,laga20214d}. 

However, the SRNF framework has several theoretical shortcomings: first, the non-injectivity of $\phi$ implies that the pullback of the $L^2$ metric by $\phi$ is degenerate. Consequently, there arises the phenomenon that distinct shapes are indistinguishable by the SRNF shape distance. This behavior was originally studied in \cite{klassen2019closed} and was further discussed in \cite{bauer2021square}, where it was shown that for each closed surface there exists a convex surface which is indistinguishable by the SRNF distance. Moreover, the image of $\Imm$ via $\phi$ is not convex, which implies that the SRNF distance is indeed only a first-order approximation of a geodesic distance function rather than a true geodesic distance on $\Imm$, i.e., the SRNF distance does not come from geodesics (optimal deformations) in $\Imm$. Furthermore, the problem of inverting the SRNF transformation $\phi$ to recover an optimal deformation in $\Imm$ from a geodesic in $L^2(S^2, \R^3)$ is highly ill-posed. 

Consequently, to overcome the theoretical challenges of the SRNF pseudo-distance, it is natural, instead, to consider the reparametrization invariant Sobolev metrics mentioned previously. In~\cite{su2020shape}, a first step towards obtaining a more general numerical framework was ach-ieved: the authors proposed a numerical framework for a family of first-order Sobolev metrics. The main drawback of this framework is the requirement for the data to be given by a spherical parametrization, which severely limits its applicability in practical contexts (see the comments below). In addition, numerical experiments suggested that it would be beneficial to consider metrics that involve higher-order terms to prevent the occurrence of certain numerical instabilities. Prior to the present paper, and to the best of the authors' knowledge, no implementation of more general (higher) order metrics was available.

The major difficulty in the implementation of these Riemannian frameworks is the discretization of the rep-arametrization group. In
\cite{laga2017numerical,su2020shape}, the authors used a discretization via spherical harmonics. These methods provide a relatively fast and stable approach for solving the registration of two spherical surfaces but requires that the surfaces are of genus-zero and are given by their spherical parametrizations. However, in most applications, data is typically given as triangular meshes that are not a priori homeomorphic to $S^2$. As the reparametrization problem is highly non-trivial and computationally expensive, a better approach for working with real data consists of developing methods that deal directly with triangular meshes.  

\begin{figure}
    \centering
    \includegraphics[width=.49\textwidth]{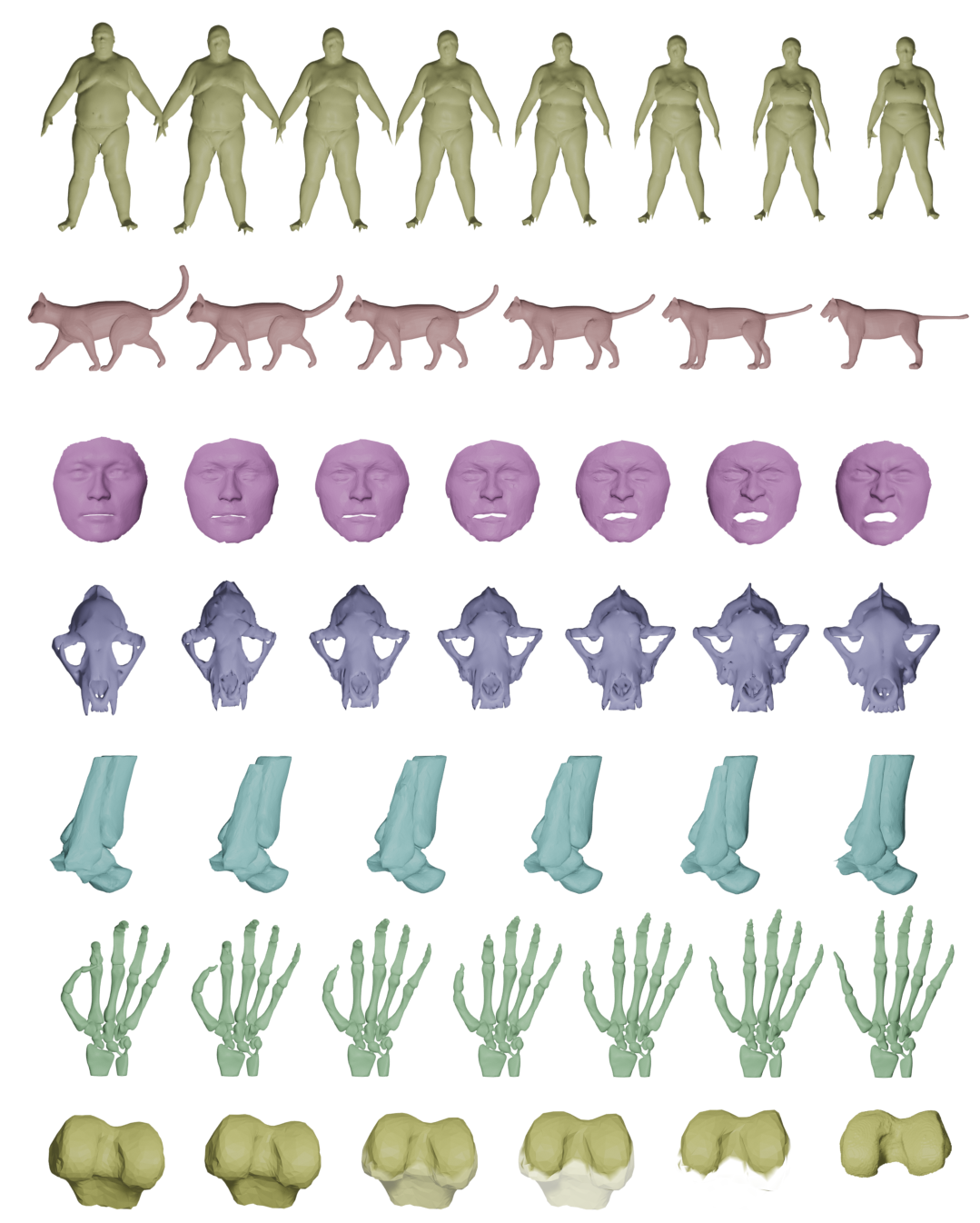}
   
\begin{tabular}{ m{1cm}m{1cm}m{1cm}m{1cm}m{1cm}m{1cm}}
\phantom{hh}source&\phantom{hh}$q(0)$&\phantom{h}$q(1/3)$&$\; q(2/3)$&\phantom{h} $q(1)$&target
\end{tabular}
    \caption{Examples of optimal deformations (geodesics) between different types of data with unknown point correspondences: genus zero surfaces (line 1 and 2), higher genus surfaces with boundaries  and inconsistent topologies (line 3 and 4), shape complexes (line 5 and 6), partial matching (line 7). Animations of the obtained surface deformations can be found in the supplementary material and on the github repository.}
        \label{fig:main}
\end{figure}

\begin{figure}
    \centering
    \includegraphics[width=.40\textwidth]{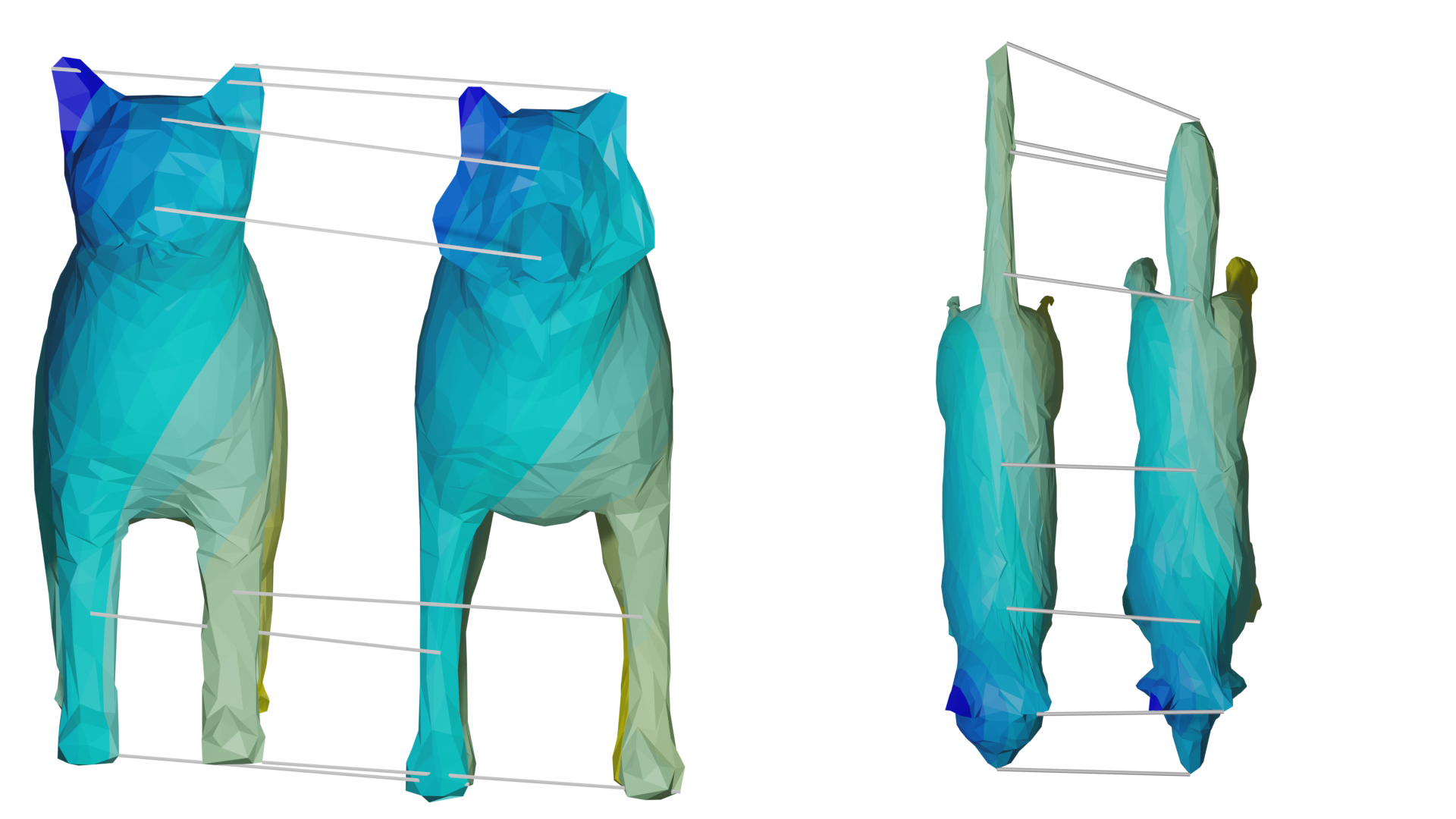}
    \caption{Point correspondences obtained after matching two unparametrized surfaces: the coloring of the two surfaces encode the obtained point correspondences. In addition, we highlight the obtained matching for selected points by displaying connecting lines.}
       \label{fig:point_correspondences}
\end{figure}

Inspired by the use of tools from geometric measure theory and in particular by varifold norms with the LDDMM model, \cite{bauer2021numerical} proposed a varifold matching framework to register surfaces with respect to the SRNF pseudo-distance. This approach provides several benefits: notably, the reparametrization group does not need to be discretized and its action on $\Imm$ does not need to be implemented. This allows one to work with simplicial meshes without having to first produce spherical parametrizations. Moreover, it extends to the analysis of surfaces with more general topologies with or without boundaries. Yet, this framework still suffers from the theoretical disadvantages of the SRNF pseudo-distance discussed above and it has been observed that the degeneracy of the distance can also lead to important numerical artefacts (c.f. Figure \ref{fig:spikes} below). Consequently, it seems natural to combine this framework with more general Riemannian metrics on $\Imm$, which is one of the main contributions of the present article, as we explain in the following section. 

\subsection{Contributions}
\label{ssec:contributions}
The central contribution of the present paper is the development of an open-source numerical framework for the statistical shape analysis of surfaces (triangular meshes) under second-order reparametrization invariant Sobolev metrics. In addition, our framework allows one to deal with topologically inconsistent and/or partially observed data.
The code is available on github:
\begin{center}
\small \url{https://github.com/emmanuel-hartman/H2_SurfaceMatch}
\end{center}
Towards this end, we extended the relaxed varifold mat-ching framework of \cite{bauer2021numerical} to compute the geodesic distance with respect to a reparametrization invariant second-order Sobolev metric on $\Imm$ and introduce a natural discretization of this metric for triangular meshes. This framework is the first implementation of higher-order Sobolev metrics on parametrized and unparametrized surfaces. 
In contrast with \cite{bauer2021numerical}, our framework directly produces geodesics (i.e. the optimal deformation path), and the addition of higher-order terms prevents the formation of numerical artefacts as mentioned above. By splitting the metric into separate terms, we are also able to control the geometric changes penalized by the metric. This allows us to model different deformations as well as control the regularizing effects of the higher-order terms, thus making our framework versatile for a variety of applications.

In addition to providing a framework for surface matching, we develop a comprehensive statistical pipeline for the computation of Karcher means, tangent principal component analysis, and parallel transport. As an application of the latter, we demonstrate how it can be used for motion transfer between surfaces. Thus, our framework is well adapted to the statistical analysis of populations of shapes such as the ones appearing in biomedical applications. To further improve the robustness of our proposed methods, we also implement a weighted varifold matching framework by extending the idea proposed in the context of curves and shape graphs by \cite{sukurdeep2021new}. The joint estimation of weights on the source surface enables this augmented model to deal more naturally with partial matching constraints or missing parts in the target shape, or differences in topology between the two shapes.


\subsection{Outline}
Our paper is structured as follows: In Section~\ref{sec:sobolev_metrics_surfaces}, we introduce the family of $H^2$-metrics (second-order Sobolev metrics) on the space of parametrized and unparametrized surfaces. In Section~\ref{sec:relaxed_matching_problem}, we formulate a varifold-based relaxed matching problem that allows us to estimate geodesics and distances induced by $H^2$-metrics on the shape space of unparametrized surfaces. We then describe a set of numerical approaches for the computation of these geodesic and distance estimates in Section~\ref{sec:numerical_optimization_approach}, before leveraging these algorithms for the development of more general tools for the statistical shape analysis of sets of surfaces in Section~\ref{sec:statistical_shape_analysis_surfaces}. Finally, we extend our second-order elastic surface analysis framework to the setting of surfaces which may have incompatible topological properties or exhibit partially missing data in Section~\ref{sec:partial_matching}.

\section{Sobolev metrics on surfaces}
\label{sec:sobolev_metrics_surfaces}
In this section, we introduce the theoretical background on second-order elastic Sobolev Riemannian metrics for spaces of parametrized and unparametrized surfaces, which will provide the key ingredient of our statistical framework for shape analysis of surfaces.

\subsection{Metrics on spaces of parametrized surfaces}
\label{ssec:parametrized_surfaces}
We begin by introducing the main definitions and known theoretical results on second-order Sobolev metrics for spaces of parametrized surfaces in $\R^3$, which we shall rely on for the remainder of this paper. Let $M$ denote a $2$-dimensional compact manifold, possibly with boundary, whose local coordinates are denoted by $(u,v) \in \R^2$. A \textit{parametrized immersed surface} in $\R^3$ is an oriented smooth mapping $q \in C^{\infty}(M, \R^3)$, which in addition, we assume to be regular, i.e., we require its differential $dq$ to be injective at every point of $M$. The set of all such parametrized surfaces, which we denote by $\I$, is itself an infinite-dimensional manifold, where the tangent space at any $q \in \I$, denoted $T_q \I$, is given by $C^{\infty}(M,\R^3)$. Any such tangent vector $h\in T_q\I$  can be thought of as a vector field along the surface $q$.

Next we introduce the \textit{reparametrization group} $\D$, which is the group of orientation-preserving diffeomorphisms of $M$, i.e., the space of all $\varphi \in C^{\infty}(M)$ such that $\operatorname{det}(D \varphi (u,v)) > 0$ for all $(u,v)$ and $\varphi^{-1} \in C^{\infty}(M)$, where $D \varphi$ denotes the differential (or Jacobian) of the diffeomorphism $\varphi$. For any immersed surface $q \in \I$ and $\varphi \in \D$, we say that $q \circ \varphi \in \I$ is a \textit{reparametrization} of $q$ by $\varphi$.

Our goal is to equip the manifold $\I$ with a Riemannian metric that will subsequently enable us to develop a framework for the comparison and statistical shape analysis of surfaces. Recall that any Riemannian metric $G$ on $\I$ induces a (pseudo) distance on this space, which is given for any two parametrized surfaces $q_0, q_1 \in \I$ by
\begin{equation}
\label{eq:geodesic_dist_parametrized}
\operatorname{dist}^G(q_0,q_1)=\underset{q(\cdot) \in \mathcal P_{q_0}^{q_1}} {\operatorname{inf}} \int_0^1\sqrt{G_{q(t)}(\partial_t q(t),\partial_t q(t))} dt,
\end{equation}
with the infimum being taken over the space of all paths of immersed surfaces connecting $q_0$ and $q_1$, which we write as:
\begin{equation}
\label{eq:pathspace}
\mathcal P_{q_0}^{q_1}:= \left\{q(\cdot) \in C^{\infty}([0, 1], \I): q(0) = q_0, q(1) = q_1\right\},  
\end{equation} 
with $\partial_t q(t)$ denoting the derivative with respect to $t$ of this path. In finite dimensions this distance, which is called the geodesic distance, is always non-degenerate, i.e., a true distance. In our infinite-dimensional setting it can, however, be degenerate~\cite{michor2005vanishing}.  

As our main goal will be the analysis of unparametrized surfaces, we will require our Riemannian metric to be invariant under the action of the aforementioned reparametrization group $\mathcal D$, i.e., we require $G$ to satisfy
\begin{equation}\label{eq:metric_inv}
G_q(h,k)=G_{q\circ\varphi}(h\circ\varphi,k\circ\varphi)    
\end{equation}
for all $q\in \mathcal I$, $h,k\in T_q\mathcal I$ and $\varphi \in \mathcal D$, which will imply that the induced geodesic distance as defined in~\eqref{eq:geodesic_dist_parametrized} satisfies
\begin{equation}
\operatorname{dist}^{G}(q_0,q_1)=\operatorname{dist}^{G}(q_0\circ\varphi,q_1\circ\varphi)
\end{equation}
for all $q_0, q_1 \in \I$ and $\varphi \in \D$. This will later allow us to consider the induced Riemannian metric (and distance function) on the quotient space of unparametrized surfaces, cf. Section~\ref{ssec:unparametrized_surfaces}.  

The simplest and potentially most natural such metric is the reparametrization invariant $L^2$-metric, which is given by
\begin{equation}
\label{eq:L2metric}
G_q(h,k)=\int_M \langle h,k \rangle \vol_q,
\end{equation}
where $\vol_q$ is the surface area measure of the immersion $q$, which in local coordinates $(u,v)$ is given by 
\begin{equation*}
\vol_q=|q_u\times q_v|dudv,
\end{equation*}
where the subscripts denote partial derivatives, $\times$ denotes the cross product on $\R^3$, and $|\cdot|$ denotes the norm on $\R^3$. This Riemannian metric is, however, not useful for any application in shape analysis, as it results in vanishing geodesic distance on both the spaces of parametrized and unparametrized surfaces~\cite{michor2005vanishing,bauer2012vanishing}. Vanishing geodesic distance refers to the phenomenon where the geodesic distance induced by the $L^2$-metric between any two surfaces is zero. 

Consequently, we are interested in stronger Riemannian metrics that induce meaningful distances. A natural approach to strengthen the metric consists of incorporating derivatives of the tangent vector, leading to the class of first-order Sobolev metrics.
Therefore we let $g_q=q^*\langle \cdot , \cdot \rangle$ be the pullback metric of the Euclidean metric on $\mathbb R^3$, see Fig.~\ref{fig:pullback} for an explanation of this construction.
\begin{figure}
    \centering
    \includegraphics[width=.5\textwidth]{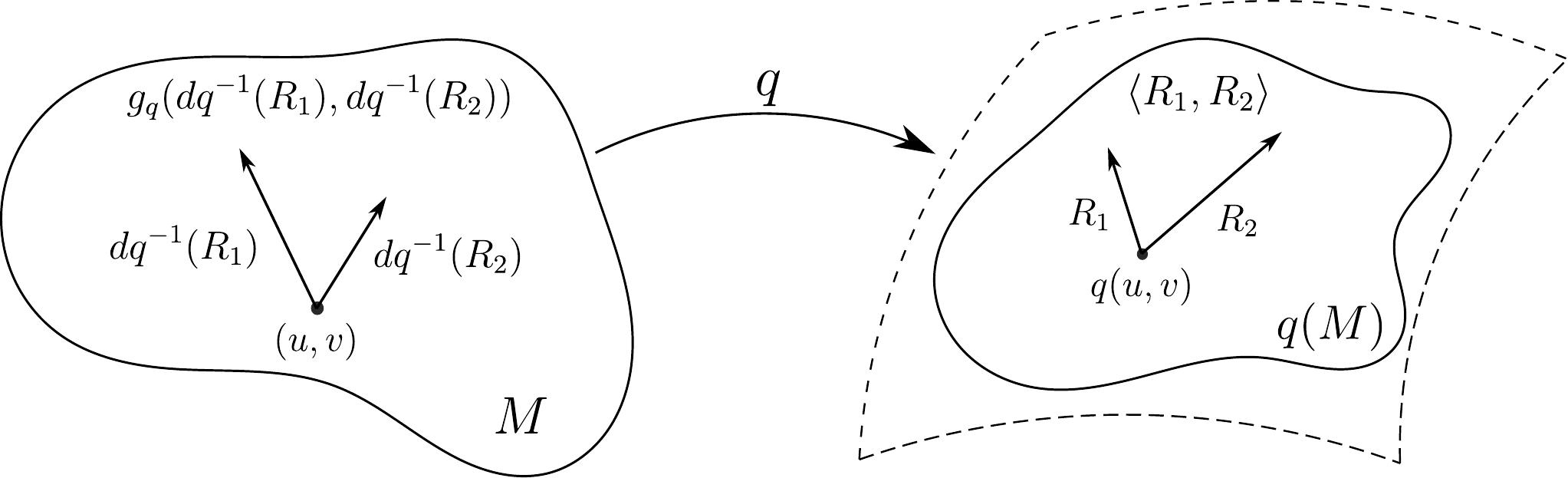}
    \caption{The induced pullback metric on $M$ of an immersion $q: M\to \mathbb R^3$.}
    \label{fig:pullback}
\end{figure}

A first-order Sobolev metric is then given by
\begin{equation}
\label{eq:H1metric}
G_q(h,h)=\int_M \langle h,h \rangle +g_q^{-1}(dh,dh)\vol_q.
\end{equation}
To interpret the first-order term $g_q^{-1}(dh,dh)$, we view the differential $dh$ as a vector valued one form, i.e., as a map from $TM$ to $\mathbb R^3$. Then the inverse of the pullback metric $g_q^{-1}$  can be used to pair such mappings. To understand this pairing better, we can fix a set of coordinates and view all the involved objects as matrix fields. Then we have 
\begin{equation}
g_q^{-1}(dh,dh)=\operatorname{tr}(dh.g_q^{-1}.dh^T), 
\end{equation}
where $dh^T$ denotes the point wise transpose of the matrix field $dh$. 
By the results of~\cite{bauer2011sobolev}, we know that this metric indeed overcomes the degeneracy of the $L^2$-metric, i.e., the corresponding geodesic distance function is non-degenerate.

Next we further decompose the first-order term into four different terms which each have a geometric interpretation. Therefore, we write
\begin{equation}
dh= dh_m+dh_++dh_\perp+dh_0,
\end{equation}
where
\begin{align*}
dh_m & =  \frac12dq g_q^{-1}(dq^Tdh+dh^Tdq) - \frac12\operatorname{tr}(g_q^{-1}dq^Tdh)dq\\
dh_+&=\frac{1}{2}\operatorname{tr}(g_q^{-1}dq^Tdh)dq\\
dh_{\perp} &= dh - dq g_q^{-1}dq^Tdh\\
dh_0 &= \frac12dq g_q^{-1}(dq^Tdh-dh^Tdq).
\end{align*}
A straight-forward calculation shows that these terms are orthogonal with respect to the inner product \begin{equation*}
\int_M g_q^{-1}(\cdot,\cdot)\vol_q,
\end{equation*} 
see \cite{su2020shape}. Consequently we have:
\begin{align*}
 &\int_M g_q^{-1}(dh,dh)\vol_q\\
 &\qquad=
\int_M g_q^{-1}(dh_m,dh_m)\vol_q +\int_M g_q^{-1}(dh_+,dh_+)\vol_q\\&\qquad+\int_M g_q^{-1}(dh_\bot,dh_\bot)\vol_q+
\int_M g_q^{-1}(dh_0,dh_0)\vol_q.
\end{align*}
The geometric meaning of the first three terms becomes clear in the following result:
\begin{remark}[Su et. al.~\cite{su2020shape}]\label{rem:firstorderterms}
\emph{Let $q\in \mathcal I$ and $h\in T_q \mathcal I$. 
The term 
\begin{equation*}
\int_M g_q^{-1}(dh_m,dh_m)\vol_q
\end{equation*}
measures the change of the pull-back metric $g_q$ while keeping the volume form constant (shearing). The second term 
\begin{equation*}
\int_M g_q^{-1}(dh_+,dh_+)\vol_q
\end{equation*}
measures the change of the volume density $\vol_q$ (scaling), while the third term 
\begin{equation*}
\int_M g_q^{-1}(dh_\bot,dh_\bot)\vol_q
\end{equation*}
measures the change in the normal vector $n_q$ (bending).}
\end{remark}
The interpretation of the last summand is less clear: it can be thought of as measuring the deformation of the local parametrization by a rotation in the parameter space $M$. 
\begin{remark}
The class of first-order Sobolev (pseudo-) metrics, i.e. metrics obtained as weighted combinations of the four first-order terms discussed above, have often been referred to as \textit{elastic metrics} in the shape analysis literature \cite{jermyn2012elastic,jermyn2017elastic}. Beyond just the mere high level analogy of these metrics measuring some form of bending or stretching energies, it turns out that these can in fact be more precisely connected to classical linear elasticity, specifically as the thin shell limit of the elastic energy of a layered isotropic material. For the purpose of concision, we will not elaborate on this particular point in this paper, but this connection will be highlighted in more details in an upcoming preprint.   
\end{remark}

The above considerations suggest that metrics of this form provide a meaningful class of metrics for shape analysis of surfaces: they overcome the degeneracy of the $L^2$-metric and admit a physical interpretation of the different terms involved. There is, however, numerical evidence that these first order metrics are still too weak for our targeted applications, see the experiments in Figure~\ref{fig:spikes}. Thus we will augment the metric with a further higher-order term involving the Laplacian $\Delta_q$ induced by the immersion $q$, which using Einstein summation is given in local coordinates $(u,v)$  by
\begin{equation*}\Delta_qh= \frac{1}{\sqrt{|g_q|}}\partial_u\left(\sqrt{|g_q|}g_q^{uv}\partial_v h \right),
\end{equation*}
where $|g_q|$ denotes the determinant of the pullback metric in the local coordinate frame. 

This allows us to define a second-order term via
\begin{equation}
\label{eq:H2term}
\int_M \langle\Delta_q h,\Delta_q h\rangle\vol_q.
\end{equation}
By adding up all the zero, first and second-order terms, we arrive at the main object of the present article: the family of $H^2$-metrics (second-order Sobolev Riemannian metrics) for surfaces, which is given by:

\begin{framed}
\begin{equation}
\label{eq:H2metric}
\begin{aligned}
&G_q(h,k)=\int_M\bigg( a_0 \langle h,k \rangle +
a_1 g_q^{-1}(dh_m,dk_m)\\&\qquad\qquad +b_1g_q^{-1}(dh_+,dk_+)+ c_1g_q^{-1}(dh_\bot,dk_\bot)\\&\qquad\qquad+ d_1 g_q^{-1}(dh_0,dk_0)
+a_2 \langle\Delta_q h,\Delta_q k\rangle\bigg)\vol_q.
\end{aligned}
\end{equation}
\end{framed}

Here $a_0, a_1, b_1, c_1, d_1, a_2$ are non-negative weighting coefficients for the different terms in the metric. Note that this family incorporates the Riemannian metric corresponding to the SRNF (pseudo) distance~\cite{jermyn2017elastic} and the families of elastic Riemannian metrics as proposed by Jermyn et. al.~\cite{jermyn2012elastic} and Su et. al. \cite{su2020shape}. For a general treatment of properties of Sobolev metrics we refer to the article~\cite{bauer2011sobolev}, and for a detailed explanation of the influence of these coefficients on our numerical experiments, see the discussion in Section~\ref{ssec:influence_constants}.

The following result, which summarizes the invariances of our family of metrics, ensures that the metric descends to quotient spaces with respect to the corresponding group actions: 
\begin{lemma}
\label{lem:invariances}
The family of $H^2$-metrics $G$ is invariant under the action of the group of reparametrizations $\mathcal D$, the group of rotations $\operatorname{Rot}(\R^3)$ and the group of translations $\R^3$, i.e., for any $q \in \I$, $h,k \in T_q \I$, $R \in \operatorname{Rot}(\R^3)$ and $\tau \in \R^3$ we have
\begin{equation}
G_q(h,k) = G_{R(q\circ\varphi)+\tau} (R(h\circ\varphi),R(k\circ\varphi)) .  
\end{equation}
It follows that  geodesic distance is also preserved by these transformations.
\end{lemma}
\begin{proof}
The invariance to the finite dimensional groups of rotations and translations follows by the fact that all the terms of the metric are invariant under this action. The invariance under the action of the infinite-dimensional group of reparametrization $\mathcal D$ follows from an application of the substitution formula for integration.\qed
\end{proof}

Having defined the class of $H^2$-metrics, we can now formulate the two main building blocks for our framework for the comparison and statistical shape analysis of surfaces: the geodesic boundary problem and the geodesic initial value problem.
\paragraph{Geodesic boundary value problem:} Given two parametrized surfaces $q_0$ and $q_1$, the geodesic boundary value problem consists in finding paths of shortest length that connect the given surfaces $q_0$ and $q_1$, i.e., calculating the geodesic distance between $q_0$ and $q_1$. Here the Riemannian length of a path $q: [0,1]\to \mathcal I$ is defined as 
\begin{equation}
 L(q):=\int_0^1\sqrt{G_{q(t)}(\partial_t q(t),\partial_t q(t))} dt.
\end{equation}
Paths of minimal length are called minimizing geodesics. By a standard result in Riemannian geometry~\cite{lang2012fundamentals}, finding minimizing geodesics is equivalent to minimizing the  Riemannian energy:
\begin{equation}
 E(q):=\frac12 \int_0^1G_{q(t)}(\partial_t q(t),\partial_t q(t)) dt.
\end{equation}
In Section~\ref{sec:numerical_optimization_approach} we will explain how to discretize this functional for discrete meshes, which will in turn allow us to solve the minimization problem using standard finite dimensional optimization methods. Note that the solution of the geodesic boundary value problem gives rise to both optimal (i.e., energy-minimizing) deformations as well as a notion of a distance between the given shapes. Thus, this operation will be the main building block of all our algorithms. 

\paragraph{Geodesic initial value problem:} While the geodesic boundary value problem searches for the shortest path between two given surfaces, the geodesic initial value problem searches for the optimal deformation path of a given surface in a given initial deformation direction. Solving the geodesic initial value problem amounts to solving the geodesic equation, which is the first-order optimality condition of the energy functional defined above. In our situation, the geodesic equation will be a  non-linear partial differential equation that is of second-order in time and fourth-order in the two-dimensional space coordinates. As this equation is rather lengthy and not particularly insightful, we refrain from formulating it here and instead refer the interested reader to the article~\cite{bauer2020fractional}, where the geodesic equation is derived for a general class of Riemanniann metrics on $\mathcal I$ that are induced by abstract pseudo-differential operators and thus include in particular the class of metrics studied in the present work. In addition,~\cite{bauer2020fractional} established local well-posedness (existence and uniqueness) of the corresponding (geodesic) initial value problem. To circumvent dealing directly with the intricacies and difficulty of solving highly non-linear and higher-order partial differential equations, we instead calculate the solution to the initial value problem using the methods of discrete geodesic calculus as developed in~\cite{rumpf2015variational}; see Section~\ref{sec:numerical_optimization_approach} for a detailed description. 

In the context of our statistical shape analysis framework for surfaces, the geodesic initial value problem will be of importance for calculating shape averages, for principal component analysis and in our motion transfer applications.

\subsection{Metrics on unparametrized surfaces}
\label{ssec:unparametrized_surfaces}
In the previous section we introduced a class of Riemannian metrics on the space of parametrized surfaces. Our main goal is, however, to compare surfaces regardless of how they are parametrized. To this end, we introduce the space of \textit{unparametrized surfaces}, which is defined as the quotient space of parametrized immersed surfaces modulo the reparametrization group, i.e., $\Shape=\I/\D$, and refer to it as shape space. This space consists of equivalence classes $[q]=\{q\circ\varphi; \varphi \in \D\}$. 

Since the family of $H^2$-metrics on $\I$ introduced in~\eqref{eq:H2metric} is reparametrization invariant, cf. Lemma~\ref{lem:invariances}, it induces a corresponding family of Riemannian metrics on the quotient space; such a construction is referred to as a Riemannian submersion, see~\cite{bauer2011sobolev} for a detailed explanation in the context of Sobolev metrics on surfaces. Consequently the geodesic distance $\operatorname{dist}^G$ given in~\eqref{eq:geodesic_dist_parametrized} corresponding to an $H^2$-metric $G$ descends to a distance function on the quotient shape space which is given as follows:
\begin{equation}
\label{eq:geodesic_dist_unparametrized}
\operatorname{dist}^{\Shape}([q_0],[q_1])=\inf_{\varphi\in \D} \operatorname{dist}^{G}(q_0,q_1\circ\varphi).
\end{equation}
By expanding the expression above, one notes that for given surfaces $[q_0]$ and $[q_1]$, computing the geodesic distance can be written as the following constrained minimization problem:
\begin{multline}
\label{eq:dist_joint_minimization}
\operatorname{dist}^{\Shape}([q_0],[q_1]) = \\
\inf_{\varphi\in \D} \underset{q(\cdot)\in \mathcal P_{q_0}^{q_1\circ\varphi}}{\operatorname{inf}} \int_0^1 G_{q(t)}(\partial_t q(t),\partial_t q(t)) dt,
\end{multline} 
where the space of paths of immersed surfaces $P_{q_0}^{q_1\circ\varphi}$ is defined in~\eqref{eq:pathspace}. In practice, computing this distance thus requires solving a \textit{matching problem} that consists of finding the optimal reparametrization and optimal path of immersions between the surfaces. We refer to the constrained minimization problem in~\eqref{eq:dist_joint_minimization} as the \textit{geodesic boundary value problem on shape space}. Compared to the matching problem for parametrized surfaces, the main difficulty in terms of numerically solving this problem consists of discretizing the action of the reparametrization group $\mathcal D$. We will circumvent this issue by introducing a relaxed version
of~\eqref{eq:dist_joint_minimization}, which will make use of methods from geometric measure theory, cf. Section~\ref{sec:relaxed_matching_problem}.

While the geodesic boundary value problem on shape space is significantly more challenging than its counterpart on parametrized surfaces, it turns out that the geodesic initial value problem for these two spaces is essentially equivalent: solving the geodesic initial value problem on the space of \textit{parametrized} surfaces for an initial condition that is in the tangent space of shape space, called a horizontal initial condition, gives rise to a solution in the space of \textit{unparametrized} surfaces. This observation follows from powerful results in Riemannian geometry and in particular from the conservation law for the horizontal initial momentum, which stems from the reparametrization invariance of the Riemannian metric~\cite{bauer2011sobolev}. Consequently, this will allow us to use the  same methods for solving the initial value problem on the space of parametrized and unparametrized surfaces, which will be described in Section~\ref{sec:numerical_optimization_approach}.

\begin{remark}
We can also consider the space of unparametrized surfaces modulo
rotations and translation. Since the class of $H^2$-metrics is also
invariant with respect to these finite dimensional group actions, cf. Lemma~\ref{lem:invariances}, it descends to a class of Riemannian metrics on this quotient space. Computing the induced geodesic distance on this quotient space involves an additional minimization over the rotation group $\operatorname{Rot}(\R^3)$ 
and over the translation group $\R^3$ in addition to minimizing over the reparametrization group and over the space of paths of immersed surfaces.
\end{remark}

\section{Relaxed matching problem}
\label{sec:relaxed_matching_problem}
We now focus our attention on the actual computation of the geodesic distance on the shape space of surfaces. As outlined in the previous section via equation~\eqref{eq:dist_joint_minimization}, this computation involves a joint optimization over paths of immersed surfaces and reparametrizations.  In practice, the space of parametrized surfaces $\I$, and hence the path of immersions, can be discretized by considering a triangular mesh as the domain $M$ of the function space $\I$, and considering piece-wise linear functions defined on $M$, which gives rise to triangulated surfaces, as outlined in~\cite{bauer2021numerical}. More general discretizations schemes, such as spline discretizations, could be used as well. Discretizing the paths of immersed surfaces implies that the minimization over those paths can be framed quite naturally as a standard finite dimensional optimization problem, as we will outline in Section~\ref{sec:numerical_optimization_approach}. However, dealing with the minimization over the infinite-dimensional reparametrization group is typically more difficult, and discretizing such a group and its action on surfaces is not straightforward. Recently, an alternative approach was proposed in~\cite{bauer2021numerical} where this minimization over reparametrizations of $q_1$ is dealt with indirectly by
instead introducing a relaxation of the end time constraint using a parametrization blind data
attachment term. Broadly speaking, this approach consists in considering the relaxed matching problem:
\begin{equation}
    \label{eq:relaxed_matching}
    \inf \left\{ \int_0^1 G_{q(t)}(\partial_t q(t),\partial_t q(t)) dt + \lambda \Gamma(q(1), q_1) \right\},
\end{equation}
where the minimization occurs over paths of immersed surfaces $q(\cdot) \in C^{\infty}([0,1], \I)$ that satisfy the initial constraint $q(0) = q_0$ only, and where $\Gamma(q(1), q_1)$ is a term that measures the discrepancy between the endpoint of the path $q(1)$ and the true target surface $q_1$, with $\lambda > 0$ being a balancing parameter. If we choose a discrepancy term $\Gamma$ that is independent of the parametrization of either of the two surfaces, then solving the relaxed problem above would yield $\Gamma (q(1),q_1) \approx 0$, which yields $q(1) \approx q_1 \circ \varphi$. Thus, this approach allows us to approximate the end time constraint in~\eqref{eq:dist_joint_minimization} without the need to explicitly model the reparametrization itself. Furthermore, this relaxed matching framework allows for inexact matching when computing the distance, which will turn out to be crucial when extending this framework to surfaces that can exhibit different topologies, such as surfaces with different genuses, and to surfaces with partial correspondences, as we shall outline in Section~\ref{sec:partial_matching}.

\subsection{Varifold representation and distance}
\label{ssec:varifolds}
We now describe how to construct the key ingredient in the relaxed model outlined above, namely, an effective and simple to compute data attachment term $\Gamma$ which gives a notion of discrepancy between unparametrized surfaces. Among different possible approaches, we will rely specifically on methods derived from geometric measure theory which have been used for that particular purpose in several past works on surface registration \cite{vaillant2005surface,charon2013varifold,feydy2017optimal,roussillon2019representation,bauer2021numerical}, see also the recent survey \cite{Charon2020fidelity}. In this paper, we adopt the framework of \textit{oriented varifolds} introduced in \cite{kaltenmark2017general}, following an approach similar to the authors’ previous works~\cite{bauer2018relaxed,bauer2021numerical,sukurdeep2021new}. We point out, however, that the majority of the present work could be adapted without much difficulty to some of the other types of data attachment terms developed in the aforementioned papers.  

Given any parametrized surface $q \in \I$, the varifold $\mu_q$ associated to $q$ is a positive Radon measure on the product space $\R^3 \times S^{2}$, where $S^2$ is the unit sphere. More specifically, $\mu_q$ is the image measure $(q,n_q)_*\vol_q$ where $n_q$ is the unit oriented normal field of $q$, and $\vol_q$ is the area form on $M$ induced by $q$. In other words, for any Borel set $B \subset \R^3 \times S^{2}$, $\mu_q(B)$ is the total area with respect to $\vol_q$ of all $(u,v) \in M$ such that $(q(u,v),n_q(u,v))$ belongs to $B$. A fundamental property is that this varifold representation does not depend on the parametrization of $q$. Namely, for any $\varphi \in \D$, one has $\mu_{q \circ \varphi}=\mu_q$, and thus it induces a well-defined mapping on the quotient space $\Shape$ which can be further shown to be an embedding of $\Shape$ into the space of varifolds. 

Then, any norm $\|\cdot\|$ on the space of varifolds should induce a distance on $\Shape$, given
for any $[q_0],[q_1] \in \Shape$ by $\|\mu_{q_0}-\mu_{q_1}\|$,
where we again emphasize that this expression does not depend on the choice of parametrizations for $q_0$ and $q_1$ in the respective equivalence classes $[q_0]$ and $[q_1]$. While there are many possible metrics that one can introduce on spaces of measures, norms defined from positive definite kernels on $\R^3\times S^{2}$ have been shown to lead to particularly advantageous expressions for numerical computations. Specifically, following the setting of \cite{kaltenmark2017general}, we consider the class of norms $\|\cdot\|_{V^*}$, where $V$ is a reproducing kernel Hilbert space of functions on $\mathbb R^3\times S^{2}$, whose kernel is of the form
\begin{equation}
\label{eq:kernel_var}
k_V(x_1,n_1,x_2,n_2) = \Psi(|x_1-x_2|) \Phi(n_1 \cdot n_2)\;, 
\end{equation} 
in which $\Psi$ and $\Phi$ are two functions defining a radial kernel on $\mathbb{R}^3$ and a zonal kernel on $S^{2}$, respectively. We discuss specific choices for $\Psi$ and $\Phi$ when presenting our numerical approach for solving the relaxed matching problem in Section~\ref{ssec:discretizing_varifold_norm}.

Following from the particular form of $\mu_{q_0}$ and $\mu_{q_1}$ as well as the reproducing kernel property in $V$, the inner product of the two varifolds in $V^*$ can be explicitly derived as:
\begin{multline}
\label{eq:norm_var}
\langle \mu_{q_0},\mu_{q_1}\rangle_{V^*}=\iint_{M \times M} \Psi\big(|q_0(u_0,v_0)- q_1(u_1,v_1)|\big)\\
\Phi\big(n_{q_0}(u_0,v_0)\cdot n_{q_1}(u_1,v_1)\big) \vol_{q_0}(u_0,v_0) \vol_{q_1}(u_1,v_1).
\end{multline}
Consequently, the squared varifold kernel distance between $\mu_{q_0}$ and $\mu_{q_1}$ is obtained as follows:
\begin{equation}
\label{eq:var_discr}
\| \mu_{q_0} - \mu_{q_1} \|_{V^*}^2 = \|\mu_{q_0}\|_{V^*}^2-2\langle \mu_{q_0},\mu_{q_1} \rangle_{V^*} + \|\mu_{q_1}\|_{V^*}^2 \;.
\end{equation} 
It can be shown that, under the right regularity and density assumptions on the kernel $k_V$ (c.f. Proposition 4 in \cite{kaltenmark2017general}), the above leads to a true distance when restricting to embedded unparametrized surfaces. However, note that there is no notion of geodesics in $\Shape$ corresponding to the varifold distance, as the straight path $(1-t) \mu_{q_0} + t \mu_{q_1}$ in $V^*$ is not associated to a corresponding path in the space of surfaces, due to the non-surjectivity of the mapping $q \mapsto \mu_q$. Yet, the squared varifold distance $\| \mu_{q_0} - \mu_{q_1} \|_{V^*}^2$ still provides a valid discrepancy term for the relaxed matching problem in~\eqref{eq:relaxed_matching}, which has the additional important advantage of being simple to discretize and evaluate numerically, as we shall detail in Section \ref{ssec:discretizing_varifold_norm}.     

Lastly, we point out that the above varifold discrepancy metrics are also equivariant to the action of rigid motions. Specifically, for any $q_0,q_1 \in \I$ and any $R\in \operatorname{Rot}(\R^3)$, $\tau\in \R^3$, we have
\begin{equation*}
\| \mu_{R q_0 + \tau} - \mu_{R q_1 + \tau} \|_{V^*}^2 = \| \mu_{q_0} - \mu_{q_1} \|_{V^*}^2, 
\end{equation*} 
which follows directly from the form of the kernel \eqref{eq:kernel_var}.

\subsection{Relaxed surface matching}
\label{ssec:relaxed_surface_matching}
The squared varifold distance is ideally suited for use as the discrepancy term $\Gamma$ in the relaxed matching problem outlined in~\eqref{eq:relaxed_matching} due to its reparametrization invariance, which finally allows us to formulate the \textit{varifold-based relaxed matching problem} for surfaces:
\begin{framed}
Given $q_0, q_1 \in \I$, we consider the variational problem:
\begin{equation}
\label{eq:relaxed_matching_varifold_asymmetric}
\inf \left\{ \int_0^1 G_{q(t)}(\partial_t q(t),\partial_t q(t)) dt + \lambda \| \mu_{q(1)} - \mu_{q_1} \|_{V^*}^2 \right\},
\end{equation}
where the minimization occurs over paths of immersed surfaces $q(\cdot) \in C^{\infty}([0,1], \I)$ that satisfy the initial constraint $q(0) = q_0$, and where $\lambda > 0$ is a balancing parameter.
\end{framed}
Note that the (relaxed) endpoint constraint $q(1)\approx q_1\circ \varphi$ for some $\varphi \in \mathcal D$ is encoded in the varifold attachment term. The interpretation of the two terms in the relaxed energy~\eqref{eq:relaxed_matching_varifold_asymmetric} is as follows: the first term (the energy of the path of immersed surfaces) measures the cost of the optimal deformation, whereas the second term is merely a data attachment term that enforces the endpoint constraint. In this relaxed surface matching framework, we refer to $q_0$ as the \emph{source}, $q_1$ as the \emph{target} and $q(1)$ as the \emph{deformed source}.

We note, however, that the model formulated in~\eqref{eq:relaxed_matching_varifold_asymmetric} is asymmetric in the sense that interchanging $q_0$ and $q_1$ will affect the obtained minimizer. Although this is a common phenomenon for relaxed optimization problems, we next propose a symmetric formulation of the varifold-based relaxed geodesic boundary value problem. To do so, we will lift the constraint of $q(0)$ being $q_0$ and instead add a second varifold-based data attachment term which measures the similarity of $q(0)$ and $q_0$:
\begin{framed}
Given $q_0, q_1 \in \I$, we consider the variational problem:
\begin{equation}
\label{eq:relaxed_matching_varifold_symmetric}
\begin{aligned}
&\inf \bigg\{ \int_0^1 G_{q(t)}(\partial_t q(t),\partial_t q(t)) dt~+ \\&\qquad\lambda_0~\| \mu_{q(0)} - \mu_{q_0} \|_{V^*}^2 + \lambda_1~\|\mu_{q(1)} - \mu_{q_1} \|_{V^*}^2 \bigg\},
\end{aligned}
\end{equation}
where the minimization is performed over paths of immersed surfaces $q(\cdot) \in C^{\infty}([0,1], \I)$, and where $\lambda_0,\lambda_1 > 0$ are balancing parameters.
\end{framed}
This symmetric formulation of the relaxed matching problem has several advantages which we will leverage in the implementation and simulations presented in the next sections. First, for $\lambda_0 = \lambda_1$, the variational problem \eqref{eq:relaxed_matching_varifold_symmetric} is indeed symmetric in the sense that for any path $t \mapsto q(t)$, the time reversed path $t \mapsto q(1-t)$ has the same energy value for the problem of matching $q_1$ onto $q_0$ and thus the value of the infimum is the same for both matching problems. More importantly, the introduction of $q(0)$ allows us to decouple the topological or mesh properties of the immersions in the path $q(\cdot)$ with those of the source shape $q_0$. As we shall explain more in details in Section \ref{sec:numerical_optimization_approach}, this allows us to select the vertex sampling and mesh structure of the surfaces in the geodesic path independently of that of the source $q_0$, which can be used to adapt the efficient multiresolution scheme of \cite{bauer2021numerical} for numerically solving the matching problem.

\section{Numerical optimization approach}
\label{sec:numerical_optimization_approach}
In this section, we will present a set of numerical approaches for solving the geodesic boundary value problem for parametrized surfaces introduced earlier in Section~\ref{ssec:parametrized_surfaces}, the varifold-based relaxed matching problem for unparametrized surfaces introduced in Section~\ref{ssec:relaxed_surface_matching}, as well as the geodesic initial value problem introduced in Section~\ref{sec:sobolev_metrics_surfaces}. Our source code is openly available at:
\begin{center}
\small {\url{https://github.com/emmanuel-hartman/H2_SurfaceMatch}}
\end{center}

First, we describe how to discretize parametrized surfaces. We will do so by considering oriented triangulated surfaces, which are also called oriented triangular meshes, that are represented by a set of vertices, edges, and faces. We view the vertices $V$ of a mesh as an ordered set of points in $\R^3$, i.e., 
\begin{equation*}
 V:=\{v_i\in\R^3|0\leq i< n\},   
\end{equation*} 
where $n$ is the number of vertices in the mesh. Occasionally we may want to view $V$ equivalently as a single point in $\R^{3n}$. The edges $E$ of a triangular mesh are subset of $\N^2$ where $(i,j)\in E$ if and only if  there is an oriented edge from $v_i$ to $v_j$. Similarly, we view the faces $F$ of a triangular mesh as a subset of $\N^3$ where $(i,j,k)\in F$ if and only if the vertices $v_i$, $v_j$, and $v_k$ make up a face in the triangular mesh such that $(v_j-v_i)\times(v_k-v_i)$ points in the direction of the oriented normal vector. Canonically, we choose to use only the representative $(i,j,k)$ of a face such that $i<j,k$.

In the context of the geodesic boundary value problem for parametrized surfaces, the relaxed matching problem for unparametrized surfaces, and the initial value problem, we are required to solve optimization problems over paths of immersed surfaces. In the discrete setting, we will solve these minimization problems by searching over paths of meshes that each lie in a solution space, $\Sol$, defined as the set of meshes with a fixed combinatorial structure, i.e., the set of meshes with a fixed number of vertices and a fixed set of edges and faces. Thus, each $q \in \Sol$ is determined precisely by the locations of the vertices and it is natural to consider $\Sol \cong \R^{3n}$. However, we can equivalently view $q \in \Sol$ as a piecewise linear surface determined exactly by the vertices. Therefore we view $q$ as the map
\begin{equation}
    \label{eq:meshasmap}
    q:\bigsqcup_{f\in F}\sigma_f^2 \to \R^3,
\end{equation} 
where for each $f\in F$, $\sigma^2_f$ is the simplex given by 
\begin{equation*}
    \sigma^2_f:= \left.\left\{\begin{bmatrix}x\\y\end{bmatrix}\in \R_{+}^2\right| x+y< 1 \right\},
\end{equation*} 
and $q$ restricted to $\sigma^2_f$ for $f=(i,j,k)$ is given by
\begin{equation*}
    q|_{\sigma_f^2}\left(\begin{bmatrix}x\\y\end{bmatrix}\right):=\begin{bmatrix}x\\y\end{bmatrix}\cdot\begin{bmatrix}v_j-v_i\\v_k-v_i\end{bmatrix}+ v_i.
\end{equation*} 
This interpretation of a mesh will prove useful for defining the geometric quantities used in the definition of the $H^2$-metric.  

\subsection{The $H^2$-metric on the space of triangular meshes}
\label{ssec:metrics_triangular_meshes}
To establish a numerical framework based on the class of $H^2$-metrics defined in \eqref{eq:H2metric}, we must first establish a discretization of each of the components that appear in its definition. The field of discrete differential geometry establishes discrete counterparts to smooth geometric quantities such as volume forms, derivatives and the Laplacian. A review of the derivations of these discrete quantities can be found e.g. in \cite{crane2018discrete}. We will either discretize these quantities per face or per vertex of a given triangular mesh depending on the context in which they will be used in our computation.  

Recall that $q\in \Sol$ is entirely  determined by the vertices $V\in\R^{3n}$. Thus, it is natural to discretize tangent vectors on the vertices of the mesh, i.e., a tangent vector $h$ is viewed as a set of vectors in $\R^3$ assigned to each vertex $v \in V$ of the mesh. Therefore, \[h:=\{h_v\in\R^3|v\in V\} \in \mathbb R^{3n}.\]

Next, we will explain how we discretize the terms that appear in the definition of the $H^2$-metric, i.e., the volume form, the pullback metric, the normal vector, and the surface Laplacian. For a graphic explanation of our discretization, we refer to Fig.~\ref{fig:ddg}.

Recall that when we view a mesh as a map $q$, as in \eqref{eq:meshasmap}, it is affine on the simplex corresponding to each face. Therefore, it is natural to discretize the first-order terms on each face. Given a face $f\in F$, where we will assume $f=(0,1,2)$ for simplicity of notation, with vertices $v_0,v_1,v_2 \in \R^3$, we have
\begin{equation*}
    dq_f=\begin{bmatrix}e_{01}\\e_{02}\end{bmatrix} \text{ where } e_{ij}=v_j-v_i.
\end{equation*}
Given a tangent vector $h$, we can compute its differential on the face $f$ as
\begin{equation*}
    dh=\begin{bmatrix}h_{1}-h_{0}\\h_{2}-h_{0}\end{bmatrix}.
\end{equation*}
Consequently a discrete version of the pullback metric $g_q$, the volume density $\vol_q$ and the normal vector $n_q$ are given by: \begin{align*}g_f&=\begin{bmatrix}
|e_{01}|^2&e_{01}\cdot e_{02}\\
e_{01}\cdot e_{02}&|e_{02}|^2
\end{bmatrix},\\
\vol_f&=\frac{1}{2}|e_{01}\times e_{02}|,\\
n_f&=\frac{e_{01}\times e_{02}}{|e_{01}\times e_{02}|}.
\end{align*}
We denote these discrete versions by $g_f$, $\vol_f$ and $n_f$ to emphasize that they are defined on the faces. 

\begin{figure}
    \centering
    \includegraphics[width=.22\textwidth]{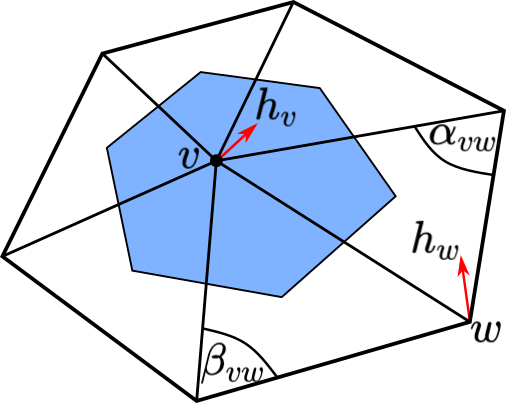}
    \includegraphics[width=.22\textwidth]{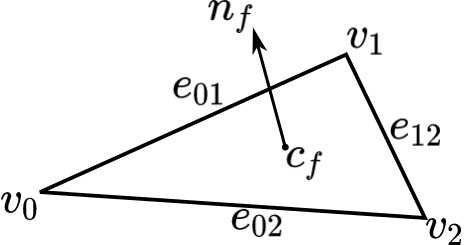}
    \caption{Defining $H^2$-metrics using discrete differential geometry. The cell dual to the vertex $v$ is shown in blue.}
    \label{fig:ddg}
\end{figure}

Given that the faces are affine, it is somewhat ``unnatural'' to discretize the Laplacian, a second-order quantity, on the faces of a mesh. Rather, the natural place to discretize the surface Laplacian is on the dual cells of a mesh. Each such dual cell corresponds to a vertex of the mesh; see Figure~\ref{fig:ddg} for an illustration. Given a tangent vector $h = \{h_v \in \R^3 | v \in V \}$ to a mesh $q = (V, E, F) \in \Sol$, the Laplacian $\Delta_q$ applied to $h$ at a vertex $v\in V$ is given by
\begin{equation*}
    (\Delta_q h)_v=\sum\limits_{\substack{w|(v,w)\in E \\\text{ or }(w,v)\in E}}(\cot(\alpha_{vw})+\cot(\beta_{vw}))(h_v-h_w).
\end{equation*}
where $\alpha_{vw}$ and $\beta_{vw}$ are angles as shown in Figure~\ref{fig:ddg}.
This discretization can be derived using either finite element methods as in \cite{crane2018discrete} or discrete exterior calculus as in \cite{crane2013digital}.

The zeroth and second-order terms of the metric contain also the volume form of our mesh, which was previously defined for each face. In order to assign this volume form to a vertex $v$ (instead of to the faces), we sum up one third of the volume of each face containing $v$. Thus, the volume form at a vertex $v$ is given by
\begin{equation*}
\vol_v=\frac{1}{3}\sum_{f|v\in f}\vol_f.
\end{equation*}
Thus we have derived discrete versions of all terms that appear in the definition of the $H^2$-metric. Therefore, given a mesh $q\in\Sol$ and a pair of tangent vectors $h,k\in T_q\Sol$ we arrive at the following expression for the discrete version of the family of $H^2$-metrics:
\begin{framed}
\begin{equation*}
\begin{split}
&G_q(h,k) = \sum_{v\in V}a_0\langle h,k \rangle \vol_v
    \\&\qquad+\sum_{f\in F}\bigg( a_1 g_f^{-1}(dh_m,dk_m)
     +b_1g_f^{-1}(dh_+,dk_+)\\&\qquad\quad+c_1g_f^{-1}(dh_\bot,dk_\bot) 
    +d_1 g_f^{-1}(dh_0,dk_0)\bigg)\vol_f\\&\qquad
    + \sum_{v\in V}a_2 \langle\Delta_q h,\Delta_q k\rangle \vol_v
\end{split}
\end{equation*}
\end{framed}


\subsection{Discretizing the $H^2$ path energy}
\label{ssec:discretizing_h2_path_energy}
Having discussed how to compute the Riemannian metric at a triangular mesh, we now explain how to discretize the Riemannian energy of a path of meshes. Indeed, given a path of triangular meshes in the solution space, denoted by $V: [0,1] \to \Sol$, we compute the path energy of $V(t)$ via
\begin{equation*}
    \int_0^1 G_{V(t)}\left(\dot{V}(t),\dot{V}(t)\right) dt,
\end{equation*}
where $\dot{V}(t)$ denotes the derivative with respect to time of the path. We re-emphasize that each mesh in the path has the same, fixed combinatorial structure, implying that the path is entirely determined by the locations of the vertices of the meshes, hence the notation $V$ above. Furthermore, we note that a further discrete approximation is required to compute the energy of the path, namely we have to discretize the time interval $[0,1]$. To that end, we consider piecewise-linear (PL) approximations for paths of meshes. Given a PL path with $N+1$ evenly spaced breakpoints $0 = t_0 < t_1 < ... <t_{N} = 1$, we can compute the tangent vector for the first $N$ points via finite differences. Thus for $i\in \{0,1,...,N-1\}$, we have
\begin{equation*}
    \dot{V}(t_i)= N(V(t_{i+1})-V(t_i)) \text{ where } t_i= \frac{i}{N}.
\end{equation*}
As a result, the energy of a PL path in $\Sol$ reduces to
\begin{equation}\label{eq:discrete_energy}
    E(V)=\frac{1}{2N}\sum_{i=0}^{N-1} G_{V(t_i)}(\dot{V}(t_i),\dot{V}(t_i)).
\end{equation}

\subsection{Solving the geodesic boundary value problem for parametrized surfaces.}
\label{ssec:numerical_solution_geodesic_bvp_parametrized}
We are now able to formulate our numerical approach for solving the geodesic boundary value problem (BVP) between parametrized surfaces. Given source and target surfaces $q_0, q_1 \in \Imm$ respectively, whose discretized versions are determined by their vertices $V_0$ and $V_1$ respectively, our goal will be to approximate solutions to the geodesic boundary value problem in $\Sol$ by minimizing the energy in~\eqref{eq:discrete_energy} over all PL paths with fixed endpoints, those being $V_0$ and $V_1$ respectively. In doing so, we have reduced the boundary value problem to a finite dimensional, unconstrained minimization problem on $\mathbb R^{3n(N-1)}$; the free variables being the vertices of the interpolating meshes between $V_0$ and $V_1$. We implement the discrete energy functional~\eqref{eq:discrete_energy} using {\tt pytorch}, which allows us to take advantage of the automatic differentiation functionality to calculate the gradient of this energy with respect to the vertices of the interpolating meshes. We then use 
the L-BFGS algorithm, as introduced in \cite{liu1989limited}, to minimize the energy. We describe this process in Algorithm \ref{alg:geodesic_bvp_parametrized} below.

\begin{algorithm}
\caption{Geodesic BVP for Parametrized Surfaces}
\label{alg:geodesic_bvp_parametrized}
\begin{algorithmic} 
\Procedure{Parametrized\_Geodesic\_BVP}{$V_0,V_1,V$} \\
\noindent  $V_0,V_1$ : vertices of the given source and target surfaces\\
\noindent  $V$ : initial guess for vertices of the interpolating meshes of the PL path 
\State $\operatorname{cost} (V)= E([V_0,V,V_1])$
\State $V = \Call{L-BFGS}{V,\operatorname{cost}}$
\State \Return $V$
\EndProcedure
\end{algorithmic}
\end{algorithm}
To speed up computations (convergence), we implemented a multi-resolution method in time, i.e., we iteratively refine the temporal discretization of the path and repeat Algorithm \ref{alg:geodesic_bvp_parametrized}, where we initialize at each iteration with an up-sampled version of the previous solution. An example of a solution to the boundary value problem for parametrized surfaces can be seen in Figure~\ref{fig:ivp_test}.

\subsection{Discretizing the varifold norm} 
\label{ssec:discretizing_varifold_norm}
In order to tackle the varifold-based relaxed matching problem for unparametrized surfaces introduced in~\eqref{eq:relaxed_matching_varifold_asymmetric}, we must discuss the discretization of the varifold data attachment term $\| \mu_{q(1)} - \mu_{q_1} \|_{V^*}^2$ introduced in Section \ref{ssec:varifolds}. We specifically need to compute the squared kernel distance between the two varifolds $\mu_{q(1)}$ and $\mu_{q_1}$ associated to the piecewise linear surfaces given by the two triangular meshes $(V(1),E_0,F_0)$ and $(V_1,E_1,F_1)$ respectively. The power of the varifold framework is that it applies equally well to this case and allows us to compare discrete shapes with significantly different mesh structures, including those with different topologies. 

Indeed, we note that an efficient discretization of the kernel inner product of~\eqref{eq:norm_var} consists in approximating the integral of the kernel over each pair of faces from $F_0$ and $F_1$ respectively by using its value at those faces' centers. In other words, we consider the following approximation: 
\begin{align*}
    &\langle\mu_{q(1)},\mu_{q_1}\rangle_{V^*}\\&\qquad\approx\sum_{f_0\in F_0}\sum_{f_1\in F_1}\Psi(|c_{f_0}-c_{f_1}|)
    \Phi(n_{f_0}\cdot n_{f_1})\vol_{f_0}\vol_{f_1},
\end{align*}
where $c_f$ denotes the barycenter of the face $f$ given by $c_f=\frac{1}{3}\sum_{v|v\in f}v$. We emphasize that the quantities $c_{f_0}$, $n_{f_0}$ and $\vol_{f_0}$ are here calculated based on the vertices $V(1)$ of the endpoint of the path of meshes, with the edges and faces $E_0$ and $F_0$ being the same as for the initial mesh in the path. The full discrepancy term $\| \mu_{q(1)} - \mu_{q_1} \|_{V^*}^2$ 
is then once again calculated as in~\eqref{eq:var_discr}, i.e., through a squared expansion of the norm induced by the kernel inner product~\eqref{eq:norm_var}, where each of the inner products is approximated as above. We emphasize that if the two meshes are exactly aligned, then the discrepancy term $\| \mu_{q(1)} - \mu_{q_1} \|_{V^*}^2$ will be minimized, while its value will be larger if the two meshes are highly misaligned.

Although several choices of kernels are available (cf. \cite{kaltenmark2017general,Charon2020fidelity} for more detailed presentations), in all the numerical simulations of this paper, we specifically chose $\Psi(|c_{f_0}-c_{f_1}|) = \exp({-\frac{|c_{f_0}-c_{f_1}|^2}{\sigma^2}})$, a Gaussian kernel of width $\sigma > 0$, for the radial kernel on $\R^3$. The value of this scale parameter $\sigma$ is typically adapted to the size of the surfaces to be matched. As for the zonal kernel on $S^2$, we take $\Phi(n_{f_0} \cdot n_{f_1}) = (n_{f_0} \cdot n_{f_1})^2$, which is known as the Cauchy-Binet kernel on the sphere. 

Since the calculation of the varifold metric involves a number of kernel evaluations that is quadratic in the number of faces, it typically represents the bulk of the numerical cost of the proposed matching algorithm. For this reason, in our implementation, we rely on the {\tt pykeops} library~\cite{charlier2021kernel}, which provides efficient GPU routines to compute such large sums of kernel functions and enables the automatic differentiation of those expressions.    

\subsection{Solving the geodesic boundary value problem for unparametrized surfaces.}
\label{ssec:numerical_solution_relaxed_matching_varifold}
Using the discretization of the $H^2$-path energy described in Section~\ref{ssec:discretizing_h2_path_energy} and the discretization of the varifold norm described in Section~\ref{ssec:discretizing_varifold_norm}, we can reduce both the non-symmetric~\eqref{eq:relaxed_matching_varifold_asymmetric} and the  symmetric~\eqref{eq:relaxed_matching_varifold_symmetric} relaxed surface matching problem to a finite dimensional, unconstrained minimization problem. Note, that free variables for the non-symmetric problem are the vertices at time $t_i$ for $i\geq 1$, while the free variables in the symmetric version include the vertices at time $t_0$. The main difference between these two algorithms is, however, that the mesh structure in the non-symmetric version is prescribed by the given data, i.e., the mesh structure (topology) in the solution space is given by the mesh structure (topology) of the source $q_0$. In the symmetric version the mesh structure of the solution is a user input and can be different from the mesh structure of both the source and the target. We describe this process below in Algorithm~\ref{alg:relaxed_match}. To speed up convergence, we implemented a multi-resolution method in both time and space, i.e., we iteratively refine the temporal  discretization of the path and the mesh discretization of the surfaces in the path and repeat Algorithm~\ref{alg:relaxed_match}, where we initialize at each iteration with an up-sampled version of the previous solution.
\begin{algorithm}
\caption{Relaxed Matching for Unparametrized Surfaces}
\label{alg:relaxed_match}
\begin{algorithmic} 
\Procedure{Relaxed\_Matching}{$V_0,V_1,V$} \\
\noindent $V_0, V_1$ : triangular meshes for the source and target.
\\$V$ : initial guess for a PL path in $\Sol$. 
\State $\operatorname{cost}(V)=\lambda_0\Call{DistVar}{V(0),V_0} + E(V)$
\State \qquad \qquad \qquad \qquad $+\lambda_1\Call{DistVar}{V(1),V_1}$ 
\State $V = \Call{L-BFGS}{V,\operatorname{cost}}$
\State \Return $V$
\EndProcedure
\end{algorithmic}
\end{algorithm}

\subsection{Solving the initial value problem}
\label{ssec:numerical_solution_ivp}
We now turn our attention to a numerical approach for solving the geodesic initial value problem (IVP) on the space of parametrized surfaces. We re-emphasize, as noted in Section~\ref{ssec:unparametrized_surfaces}, that the geodesic initial value problem on the spaces of parametrized and unparametrized surfaces are essentially equivalent. As a result, the procedure we describe in this section gives rise to a solution in the space of unparametrized surfaces as well.

To solve the geodesic initial value problem, we follow the variational discrete geodesic calculus approach developed in~\cite{rumpf2015variational}. Given a surface $q\in \Sol$ and a tangent vector $h\in T_q\Sol$ (which is assumed to be horizontal for unparametrized surfaces) our method involves approximating the geodesic in the direction of $h$ with a PL path $V$ having $N+1$ evenly spaced breakpoints. To simplify notation, we will denote surfaces in the PL path at time $t_i = \frac{i}{N}$ for $i = 0,\ldots,N$ by $V(t_i) := V_i$. At the first step, we set $V_0=q$ and $V_1=q+\frac{1}{N}h$, and find $V_2$ such that $V_1$ is the geodesic midpoint of $V_0$ and $V_2$, i.e., we solve for $V_2$ such that 
\begin{equation*}
    V_1 = \underset{\Tilde{V}}{\argmin} [G_{V_0}(\Tilde{V}-V_0,\Tilde{V}-V_0)+G_{\Tilde{V}}(V_2-\Tilde{V},V_2-\Tilde{V})].
\end{equation*} 
Differentiating with respect to $\Tilde{V}$ and evaluating the resulting expression at $V_1$, we obtain the system of equations
\begin{multline}
\label{eq:first_order_condition_ivp}
    2G_{V_0}(V_1-V_0,B_i)-2G_{V_1}(V_2-V_1,B_i)\\+ D_{V_1}G_{\cdot}(V_2-V_1,V_2-V_1)_i=0,
\end{multline}
where $B_i$ is the $i$-th basis vector of $\R^{3n}$. We denote the system of equations in~\eqref{eq:first_order_condition_ivp} by $F(V_2; V_1, V_0) = 0$, where we stress again that $V_0$ and $V_1$ are here fixed. We solve this system of equations for $V_2$ using a nonlinear least squares approach, i.e., by computing
\begin{equation*}
    V_2 = \underset{\Tilde{V}}{\argmin} \| F(\Tilde{V}; V_1, V_0) \|_2^2
\end{equation*}
via the L-BFGS algorithm, where we again take advantage of the automatic differentiation capabilities of {\tt pytorch} in our implementation. We then iterate this process step by step to compute $V_3,V_4,\ldots,V_{N}$. We summarize our approach via the pseudocode in Algorithm~\ref{alg:ivp}. An example of a solution to the initial value problem can be seen in Figure~\ref{fig:ivp_test}. These results show excellent consistency between the solutions of the corresponding boundary and initial value problems.

\begin{algorithm}
\caption{Geodesic Initial Value Problem}
\label{alg:ivp}
\begin{algorithmic} 

\Procedure{Geodesic\_IVP}{$q,h,N$} \\
\noindent $q$ : a surface in $\mathfrak{M}$ \\
\noindent $h$ : a tangent vector in $T_q \mathfrak{M}$ \\
\noindent $N$ : number of time steps
\State Set $V_0 = q$ and $V_1 = q + \frac{1}{N} h$ 
\For{$t = 2,\ldots,N$}
    \State $V_t = \underset{\Tilde{V}}{\argmin} \| F(\Tilde{V}; V_{t-1}, V_{t-2}) \|_2^2$
\EndFor
\State \Return $V = [V_0,V_1,\ldots,V_N]$
\EndProcedure

\end{algorithmic}
\end{algorithm}

\begin{figure}
    \centering
    \includegraphics[width=.49\textwidth]{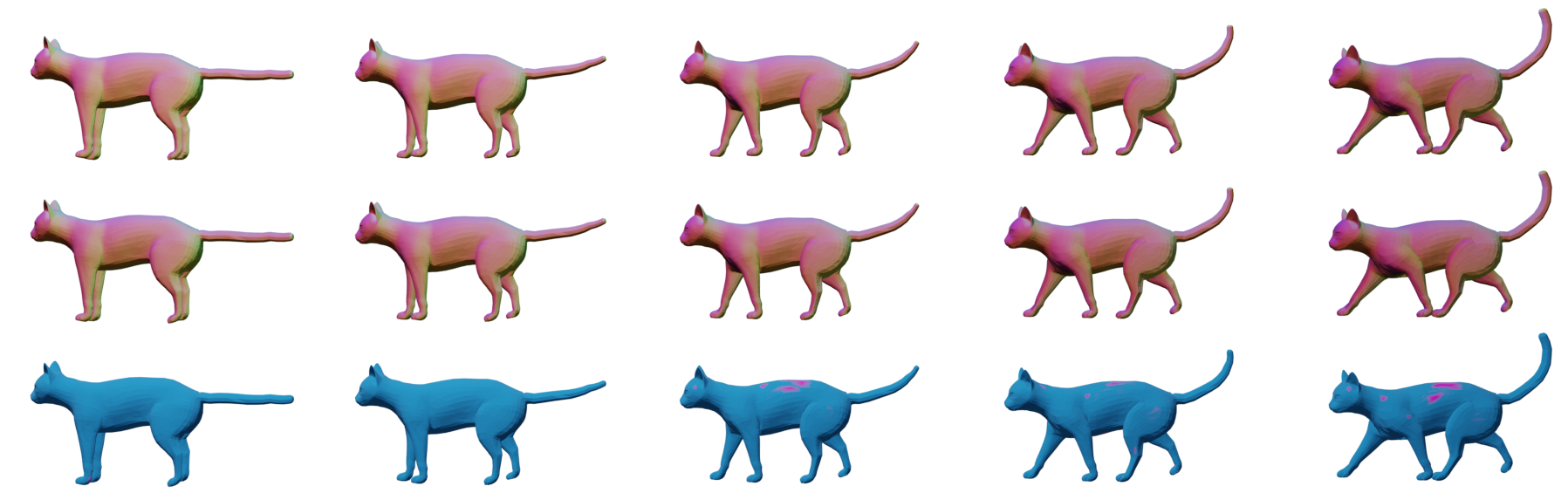}
    \caption{\small Solution to a parametrized BVP (top) and to the corresponding IVP (middle), i.e., after solving the BVP, we calculated the corresponding initial velocity of the solution and used this as the initial condition to solve the IVP. The results are overlaid (bottom) to illustrate the small discrepancy in the solutions.}
    \label{fig:ivp_test}
\end{figure}

\subsection{Influence of the metric coefficients}
\label{ssec:influence_constants}

\begin{figure}
    \centering
    \includegraphics[width=.5\textwidth]{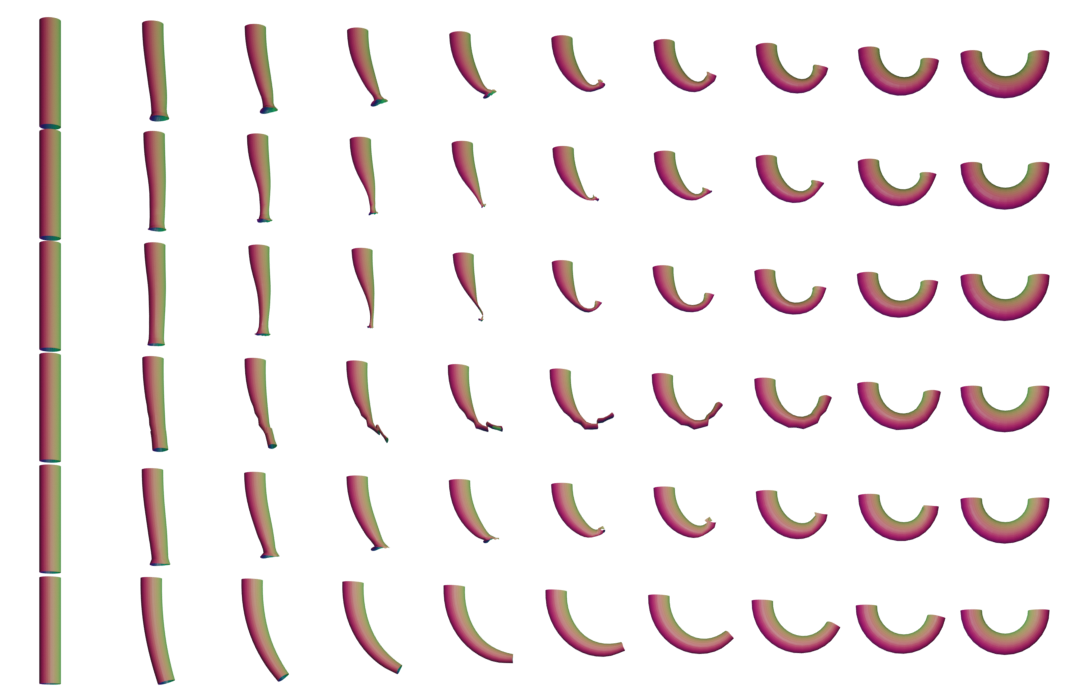}
    \caption{Influence of constants. An example of the same boundary value problem with different choices for the $H^2$-metric coefficients $(a_0,a_1,b_1,c_1,d_1,a_2).$ First row: $(1,1,1,1,1,1)$, Second row: $(10,1,1,1,1,1)$, Third row: $(1,1,1,1,1,0.1)$, Fourth row: $(1,10,10,1,1,0.1)$, Fifth row: $(1,1,10,0,1,10)$, Sixth row: $(1,100,1,1,1,1)$. }
    \label{fig:constants}
\end{figure}

\begin{figure*}
    \centering
        \begin{tabular}{m{.65cm}m{2.275cm}m{2.15cm}m{2.1cm}m{2.15cm}m{2.25cm}m{2.6cm}}
            &\phantom{hh}source&\phantom{hh}$q(0)$&\phantom{h}$q(1/3)$&\phantom{h}$q(2/3)$&\phantom{h} $q(1)$&target
        \end{tabular}
    \includegraphics[width=.9\textwidth]{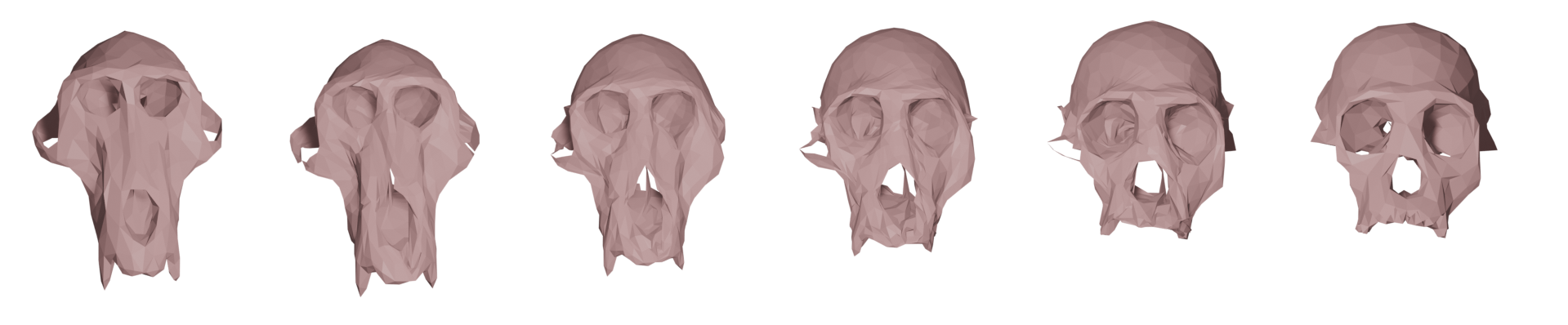}
    \includegraphics[width=.75\textwidth]{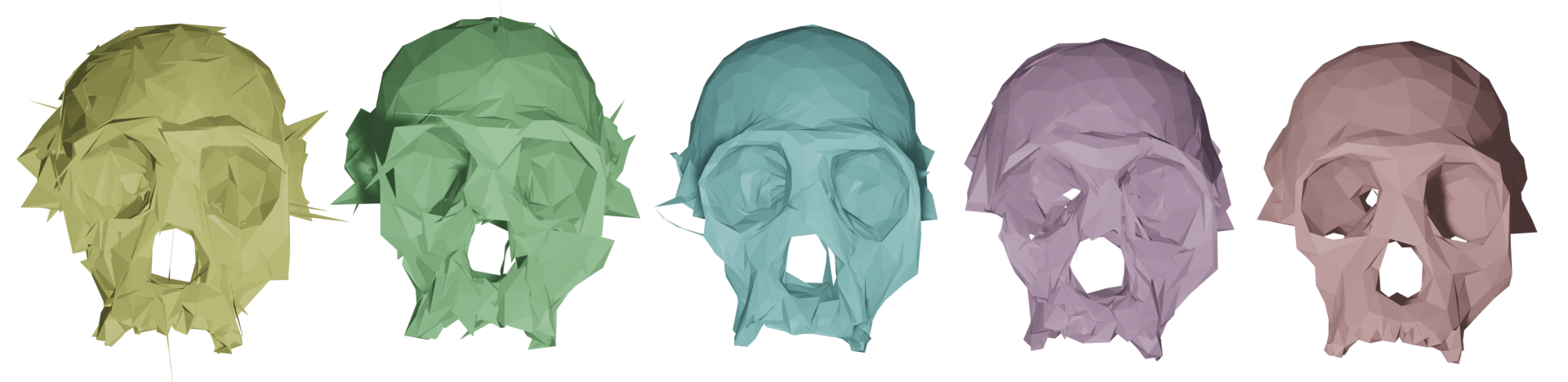}
    \caption{Matching of two skulls with highly incompatible topology. Top row: Geodesic w.r.t. to an $H^2$-metric with coefficients: $(1,1,1,1,1,2)$. Bottom row: the deformed source $q(1)$ for different metrics and methods: the SRNF pseudo distance obtained with the code of~\cite{bauer2021numerical} (yellow), an $H^1$-metric with coefficients: $(1,1,1,1,1,0)$ (green), an $H^2$-metric with coefficients: $(1,1,1,1,1,2)$ (turquoise), an $H^2$-metric with coefficients allowing for partial matching: $(1,1,1,1,1,2)$ (violet). The target is displayed on the right. One can observe the regularizing effect of the second-order terms (turquoise and violet) and, in addition, how topological inconsistencies (such as the thin arc near the left ear) are correctly removed in the partial matching framework (violet) instead of getting shrunk to almost zero volume (turquoise).}
    \label{fig:spikes}
\end{figure*}

In this section we present examples detailing the influence of the choice of constants in the $H^2$-metric on the geodesics obtained after matching parametrized and unparametrized surfaces via Algorithm~\ref{alg:geodesic_bvp_parametrized} and Algorithm~\ref{alg:relaxed_match} respectively. We also report the influence of the constants on the corresponding computation times, see Table~\ref{tab:timing}.
\begin{table*}[h!]
\centering

    \begin{tabular}{|c||c|c|c|c||c|c|c|c|}\hline
    &\multicolumn{4}{|c||}{Unparametrized BVP}&\multicolumn{4}{|c|}{Parametrized BVP}\\
     \# of vertices& $H^2$ & $H^1$ & SRNF &LDDMM&$H^2$&$H^1$&SRNF&LDDMM \\\hline
     50   &0.14s&0.11s&0.08s&0.11s&0.08s&0.07s&0.05s&0.04s\\
     200  &0.15s&0.12s&0.09s&0.12s&0.08s&0.07s&0.06s&0.04s\\
     800  &0.17s&0.13s&0.10s&0.14s&0.09s&0.08s&0.07s&0.05s\\
     3200 &0.23s&0.21s&0.17s&0.27s&0.13s&0.11s&0.08s&0.06s\\
     12800 &1.39s&0.67s&0.55s&1.12s&0.30s&0.28s&0.21s&0.15s\\
     51200&6.99s&3.88s&3.73s&14.70s&0.73s&0.69s&0.59s&1.12s\\\hline
\end{tabular}
\caption{Time per iteration (in seconds) of L-BFGS optimization for solving parametrized and unparametrized boundary value problems with respect to first order ($H^1$) and second order ($H^2)$ Sobolev metrics as well as the LDDMM diffeomorphic model (see next section), using meshes sampled with different numbers of vertices. All experiments are run on an Intel 3.2 GHz CPU with a Gigabyte GeForce GTX 2070 1620 MHz GPU.}
\label{tab:timing}
\end{table*}

A synthetic example of a geodesic boundary value problem for a variety of choices of constants can be seen in Figure~\ref{fig:constants}. We note that the zeroth-order term weighted by $a_0$ corresponds to the invariant $L^2$-metric and penalizes how far the vertices move weighted by their corresponding volume forms. In Figure \ref{fig:constants} on the second row, we see an example where $a_0$ dominates the other coefficients and as a result, the further a vertex moves, the more shrinking we observe for faces incident to these vertices. The second-order term weighted by $a_2$ penalizes paths through meshes with high local curvature. In the third line of Figure \ref{fig:constants}, we present a path where $a_2$ is chosen to be small relative to the other coefficients and as a result the midpoints of the geodesic with respect to this choice of coefficients have points with higher local curvature. As noted in Remark~\ref{rem:firstorderterms}, the terms corresponding to the weighting coefficients $a_1,b_1,$ and $c_1$ measure the shearing of faces, stretching of faces, and the change in the normal vector, respectively. In the fourth row of Figure \ref{fig:constants}, we choose $a_1$ and $b_1$ to be large and $a_2$ to be small. As a result, the geodesic with respect to this choice of coefficients passes through meshes where portions of the pipe are flattened, which produces vertices with higher local curvature without shearing or stretching the faces of the mesh.

In Figure \ref{fig:spikes}, we highlight the importance of the second order term for complex matching problems. In this figure we consider a matching problem between two surfaces undergoing strong deformations, which in addition, have inconsistent topologies. In previous work of two of the authors~\cite{bauer2020srnfmatch}, the same example has been considered for the SRNF pseudo-distance: in this framework the obtained result admitted significant singularities in the form of thin spikes that were appearing in areas of high deformations, cf. the yellow skull in Figure~\ref{fig:spikes}. We repeated this experiment using the metrics implemented in this article; in the second figure (the skull in green) one can see the resulting endpoint of an $H^1$-metric. While the resulting match is slightly superior to the one of the SRNF framework, it still exhibits some of the spike singularities. A theoretical explanation for the appearance of these singularities can be found in the observation that the $H^1$-metric is not strong enough to control the $L^{\infty}$-norm -- by the Sobolev embedding theorem the $H^1$-metric is exactly at the critical threshold. Consequently small areas can move far with a limited cost, which can potentially lead to these spike type singularities. This observation suggests that this behavior should not occur for higher-order metrics and, indeed, this is also reflected in our experiment: the turquoise skull, which was obtained using an $H^2$-metric, does not exhibit any spike singularities and leads to an overall superior matching. Note, that the thin arc in the right ear region of the skull is not a spike, but stems from the inconsistent topology of the shapes, cf. the arc at the right ear of the animal skull. We tackle these topological inconsistencies in Section~\ref{sec:partial_matching}, where we introduce a partial matching framework which would automatically erase such regions.


\begin{remark}
With our relaxed matching framework, one can obtain an adequate matching even when the meshes under consideration are of low quality, i.e., even if they include a certain level of degradation caused by topological or geometric noise, such as the presence of holes or degenerate triangles for instance. Indeed, our approach avoids the need for an exact matching of the source and target meshes (thanks to the varifold relaxation term in our variational formulation of the matching problem), which helps us avoid instances where enforcing an exact matching would lead to e.g. over-fitting parts of the source to noisy parts in the target, thus leading to inaccurate results; see the experiments in~\cite{sukurdeep2019inexact} for an illustration in the context of planar curves. Moreover, in our framework, the mesh structure for solutions to the geodesic boundary value problem is user-defined, i.e., meshes in the geodesic path can be prescribed to have any desired topology or resolution (independently of the mesh structure and resolution of the boundary shapes (i.e., the source and/or target)). In the case where the mesh quality of the boundary shapes is low, issues might arise as the varifold term is most sensitive to large discrepancies in the match. Yet, with proper initialization, this issue can be mitigated and we still obtain desirable results. This is depicted in Figure~\ref{fig:spikes}, where we are able to match two skulls despite the presence of topological noise in the data.
\end{remark}

\subsection{Comparison with other shape matching frameworks}
\label{ssec:comparison_with_other_shape_matching_frameworks}

While the previous section outlines how our approach for surface matching compares with other intrinsic Riemannian frameworks based on the SRNF and $H^1$-metrics, in this section, we present a (theoretical) comparison of our approach with a wider class of numerical frameworks for surface matching. For an overview on a variety of surface matching frameworks, we refer the interested reader to the survey article~\cite{biasotti2016recent}. This survey includes, amongst others, an overview of methods based on the metric (measure) space approach for shape matching, where one considers geometric objects (e.g. point clouds or meshes) as metric spaces (possibly with a probability distribution defined on them), and in which objects are compared via Gromov-Haussdorf distances, or via its extensions like the Gromov-Wasserstein distance~\cite{memoli2011gromov}. Here, we will focus mainly on extrinsic Riemannian models for shape analysis~\cite{beg2005computing, younes2010shapes}, the functional map framework~\cite{ovsjanikov2012functional, ren2018continuous} and the related evolutionary non-isometric geometric matching (ENIGMA) approach~\cite{edelstein2019enigma} for obtaining shape correspondences, as well as shape interpolation methods based on as-isometric-as-possible or as-rigid-as-possible deformations~\cite{kilian2007geometric}, Hamiltonian dynamics~\cite{eisenberger2020hamiltonian} or thin shell models~\cite{iglesias2018shape}, and finally deep learning based approaches for shape registration~\cite{trappolini2021shape, cosmo2020limp, huang2021arapreg}. 

First, we discuss the Large Deformation Diffeomorphic Metric Mapping (LDDMM) framework of~\cite{beg2005computing, younes2010shapes}, which stands as the main alternative Riemannian framework for shape analysis. In contrast to the intrinsic metrics considered in this paper, the LDDMM model consists in building shape metrics extrinsically via right-invariant metrics on a given subgroup of $\Diff(\R^3)$, the group of diffeomorphisms of the ambient space. Then the distance and geodesic between two surfaces is essentially computed by looking for a diffeomorphism of the whole space that warps the source onto the target surface while minimizing the kinetic energy as defined by the metric on $\Diff(\R^3)$. There are several fundamental differences between the intrinsic and extrinsic frameworks, which have already been emphasized in previous publications \cite{bauer2018relaxed,bauer2019metric}. To give a brief summary of those differences, a first important distinction is that the LDDMM approach imposes more constraints on the regularity of the surface transformation as it must be induced by a smooth deformation of the ambient space itself. A direct consequence is that this approach guarantees diffeomorphic evolution of the source shape and thus prevents the formation of singularities or self-intersection along geodesics, which is in general not the case with the intrinsic $H^2$-metric framework (see \cite{bauer2018relaxed} for examples of such phenomenon in the space of curves). On the other hand, geodesics in the diffeomorphic model are only well-defined between surfaces that belong to the same orbit for the action of the specific subgroup of diffeomorphisms and thus relaxing the problem using e.g. varifold distances is a necessity in practice. From a numerical point of view, LDDMM registration is also an optimal control problem, and it is typically solved based on a geodesic shooting scheme \cite{vialard2012diffeomorphic}. The Hamiltonian dynamical equations generally require evaluating kernel functions between all pairs of vertices in the source surface. Thus the complexity for the integration of these systems is quadratic in the number of vertices, which is an important difference with the linear complexity one gets with intrinsic metrics. This implies that for surfaces with a large number of vertices, the complexity of each iteration of the matching optimization scheme is dominated by the computation of the varifold term and its gradient in the intrinsic framework of this paper while it becomes dominated by the integration of the Hamiltonian system in the case of LDDMM. This is illustrated in Table \ref{tab:timing} that shows the time per iteration of the optimization algorithm for the different models.

%

Another popular method for the analysis of unregistered surfaces is the functional map framework~\cite{ovsjanikov2012functional, ren2018continuous}, which allows one to find optimal maps (optimal point-to-point correspondences) between pairs of surfaces by finding optimal pairings between real-valued functions defined on the surfaces. This is done by using a least squares approach to solve a linear system based on the Laplace-Beltrami operator and Wave (or Heat) Kernel Signature descriptors. These quantities describe the local geometry of surfaces and the method is largely successful at matching regions with similar local geometries. However the global matching of the framework benefits significantly from a good prior selection of landmarks, and extensions of the method, such as the evolutionary non-isometric geometric matching (ENIGMA) approach~\cite{edelstein2019enigma}, have been proposed to obtain point-to-point correspondences in a fully automatic way, even in the context of surfaces with different topologies. Moreover, such methods struggle with topological or geometric noise, such as the presence of holes, degenerate triangles or thin spikes in the meshes, which is a difficulty that our intrinsic $H^2$-metric framework handles well, as demonstrated in Figure~{\ref{fig:spikes}}. Furthermore, methods for shape matching using optimal transport techniques have been proposed. In particular, the Gromov-Wasserstein distance for object matching \cite{memoli2011gromov} treats surfaces as metric-measure spaces and solves for a probabilistic coupling that preserves pairwise distances of points in the metric-measure spaces. From such couplings, one can produce approximate point-to-point correspondences of metric measure spaces. The computation of pairwise distance matrices limits the effectiveness of these methods for high resolution meshes as these computations are quadratic with respect to the number of vertices. Furthermore, approaches like functional maps, Gromov-Wasserstein, and ENIGMA allow us to obtain approximate point-to-point correspondences, but they do not provide an optimal deformation between the registered shapes. Thus the optimal deformations have to be calculated in a second (independent) post processing step, using e.g. the as-isometric-as-possible, as-rigid-as-possible framework of \cite{kilian2007geometric}, the Hamiltonian dynamics method given by \cite{eisenberger2020hamiltonian} or a thin shell model as presented in~\cite{iglesias2018shape}. Consequently, in this setup, the registration, deformation, and statistical shape analysis are performed separately, which has been shown to be less desirable as it can introduce a significant bias in the resulting statistical analysis~\cite{srivastava2016functional}. As discussed in Section~\ref{sec:introduction}, one major advantage of Riemannian frameworks, including intrinsic frameworks like the one presented in this paper or extrinsic frameworks like LDDMM, is that they do not suffer from this shortcoming as the registration, geodesic interpolation and statistical analysis are all performed under the same metric setting.


More recently, several deep learning methods for the registration and analysis of surfaces have emerged~\cite{trappolini2021shape, cosmo2020limp, huang2021arapreg}. Such methods are exciting as they can provide significant computational gains by allowing one to register, or even interpolate between surfaces, via a simple forward pass through a pre-trained network, which can be highly desirable in practice especially when working with very large and high dimensional datasets of surfaces. Nevertheless, the quality of point-to-point correspondences or optimal deformations obtained via these deep learning methods relies on having access to a very large database of ground truths to train the network, which in practice is difficult and costly to obtain. As a result of this lack of good training data, these deep learning methods are thus susceptible to having poor generalization capabilities, resulting in situations where the method performs poorly on data that is significantly different or of significantly worse quality than the training data. One potential future application of our intrinsic $H^2$-metric framework is that it could be used to \textit{generate} high quality training data (in the form of geodesic distances, point-to-point correspondences or optimal deformations) for deep learning methods for the analysis of surfaces. Such ideas have recently been introduced in the case of functional data~\cite{nunez2021srvfnet, chen2021srvfregnet} and in the setting of planar curves~\cite{hartman2021supervised, nunez2020deep}, where early results have been encouraging.

\section{Statistical shape analysis of surfaces}
\label{sec:statistical_shape_analysis_surfaces}
Beyond the comparison of two surfaces, the mathematical and numerical framework developed in the previous sections can be used as building blocks for the development of more general tools for the statistical shape analysis of sets of surfaces. In this section, we discuss in particular how to extend our approach to calculate sample averages, perform principal component analysis, and approximate parallel transport between parametrized and unparametrized surfaces. Our implementation of these statistical shape analysis methods are available in our open source code\footnote{\url{https://github.com/emmanuel-hartman/H2_SurfaceMatch}}.

\subsection{Karcher mean} 
\label{ssec:karcher_mean}
A central tool in any statistical shape analysis toolbox is the notion of a Karcher mean. Let $(\mathcal{M},G)$ be a (possibly) infinite-dimensional Riemannian manifold with corresponding geodesic distance function $\operatorname{dist}^G$. Given data $x_1,\dots,x_K\in \mathcal M$, the Karcher mean $\bar x$ is the minimizer of the sum of squared distances to the given data points, i.e., 
\begin{equation}
    \label{eq:karcher_mean_functional}
    \Bar{x} = \underset{x \in \mathcal{M}}{\argmin} \sum_{k=1}^K \dist^G(x, x_k)^2.
\end{equation}
Note that the existence and uniqueness of the Karcher mean is a priori not guaranteed, but requires that the data is sufficiently concentrated, i.e., belongs to a ball in the geodesic distance whose radius depends on the curvature of the manifold $\mathcal M$. The Karcher mean can be computed by a gradient descent method, as proposed e.g. in~\cite{pennec2006intrinsic}. This method requires the computation of $K$ geodesic boundary value problems at each gradient step, where usually a relatively large number of iterations (gradient steps) is necessary. 
 
For computational efficiency, we instead implemen-ted an alternative algorithm to approximate the Karcher mean based on the iterative geodesic centroid procedure as proposed e.g. in~\cite{ho2013recursive}: Given data points $x_1,\ldots x_K \in \mathcal{M}$ and an initial guess $x_0$, we generate a sequence of estimates for the Karcher mean, namely $\hat x_{i}$ for $i = 0,\ldots,N_{\operatorname{iter}}$ where $N_{\operatorname{iter}} = O(K)$, by setting $\hat x_0 = x_0$, and iteratively defining $\hat x_i = x(1/(i+1))$, with $x(t)$ being the geodesic connecting $\hat x_{i-1}$ to a data point $x_{k}$ which has been uniformly chosen at random (with replacement) from the dataset. Thus one only has to calculate $N_{\operatorname{iter}} = O(K)$ geodesics \emph{in total}, which is linear in the number of data points.

A pseudo-code of this method is presented in Algorithm~\ref{alg:karcher_mean}. For parametrized surfaces, we can initialize this algorithm with the Euclidean mean (average) of the vertices of the surfaces in our sample, assuming of course that they have been centered. We then iteratively solve the geodesic boundary value problem for parametrized surfaces using Algorithm~\ref{alg:geodesic_bvp_parametrized}, whose inputs at the $i^{th}$ iteration are the current Karcher mean estimate $\Bar{V}$ as the source, the randomly chosen surface $V_k$ as the target, and a linearly interpolated path between the source and target as initial guess for the PL path $V$. For unparametrized surfaces, there is the additional difficulty that the data might have inconsistent mesh structures. In order to extend the computation of the Karcher mean to this situation, one needs to initialize the Karcher mean estimate $(\Bar{V}, \Bar{E}, \Bar{F})$ to some user-defined mesh, which will determine the connectivity and topology of the Karcher mean, and then iteratively solve the relaxed matching problem for unparametrized surfaces using Algorithm~\ref{alg:relaxed_match}. Note that, as inputs for the relaxed matching problem at the $i^{th}$ iteration, we can use the current Karcher mean estimate $(\Bar{V}, \Bar{E}, \Bar{F})$ as the source, the randomly chosen surface $(V_{k}, E_{k}, F_{k})$ as the target, and a constant path of the Karcher mean estimate for the initial PL path $V$. An example of a population of unparametrized shapes
can be seen in Figure~\ref{fig:PCA_unparam_faces}, together with their Karcher mean which has been computed via Algorithm~\ref{alg:karcher_mean}.

\begin{algorithm}
\caption{Karcher Mean}
\label{alg:karcher_mean}
\begin{algorithmic} 

\Procedure{Parametrized\_Karcher\_Mean}{$V_1,\ldots,V_K$} \\
\noindent $V_1,\ldots,V_K$ : vertices of parametrized surfaces from sample
\State Initialize $\Bar{V} = \frac{1}{K} \sum_{k=1}^K V_k$ 
\For{$i = 1,\ldots,N_{\operatorname{iter}}$}
    \State $V_k \sim \operatorname{Unif}\{V_1, \ldots, V_K\}$ 
    \State $V = \Call{Linear\_Interpolation}{\Bar{V}, V_{k}}$
    \State $V = \Call{Parametrized\_Geodesic\_BVP}{\Bar{V}, V_{k}, V}$
    \State $\Bar{V} = V(1/(i+1))$
\EndFor
\State \Return $\Bar{V}$
\EndProcedure

\Procedure{Unparametrized\_Karcher\_Mean}{$V_1,...,V_K$, $\bar V$} \\
\noindent $V_1,\ldots,V_K$ : triangular meshes from the sample \\
$\bar V$ : initial guess for Karcher mean
\For{$i = 1,\ldots,N_{\operatorname{iter}}$}
    \State $V_k \sim \operatorname{Unif}\{V_1, \ldots, V_K\}$ 
    \State $V = \Call{Linear\_Interpolation}{\Bar{V}, \Bar{V}}$
    \State $V= \Call{Relaxed\_Matching}{\Bar{V}, V_{k}, V}$
    \State $\Bar{V} = V(1/(i+1))$ 
\EndFor
\State \Return $\Bar{V}$
\EndProcedure
\end{algorithmic}
\end{algorithm}

\subsection{Dimensionality Reduction}
\label{ssec:dimensionality_reduction}
Dimensionality reduction 
is a key tool in modern statistics and machine learning. We illustrate how to construct two popular dimensionality reduction tools for the statistical shape analysis of surfaces using our framework, namely data visualization through multidimensional scaling, and principal component analysis.

\subsubsection{Visualizing the distance matrix using multidimensional scaling}
\label{ssec:cmds}
Multidimentional scaling (MDS) is a well-known procedure for mapping $K$ points in a high (or infinite) dimensional space into a lower dimensional space, while maintaining information about the pairwise distances between these $K$ points. More specifically, given a (possibly) infinite-dimensional Riemannian manifold $(\mathcal{M},G)$ with corresponding geodesic distance function $\operatorname{dist}^G$, with data $x_1,\dots,x_K\in \mathcal M$, the goal of MDS is to find points $\Hat{x}_1, \ldots, \Hat{x}_K \in \R^d$ for some $d > 0$ such that:
\begin{multline*}
    \Hat{x}_1, \ldots, \Hat{x}_K = \\
    \underset{y_1,\ldots,y_K \in \R^d}{\argmin} \left( \sum_{i \neq j}  (\operatorname{dist}^G(x_i, x_j) - \|y_i - y_j\|)^2 \right)^{\frac{1}{2}}.
\end{multline*}
In the context of statistical shape analysis, one can use MDS to project a dataset of surfaces as points in Euclidean space for data visualization purposes, or as an intermediary step in clustering applications with sets of surfaces, see Figure~\ref{fig:mds}.

\begin{figure}
    \centering
    \includegraphics[trim = 0mm 0mm 0mm 0mm ,clip,width=.3\textwidth]{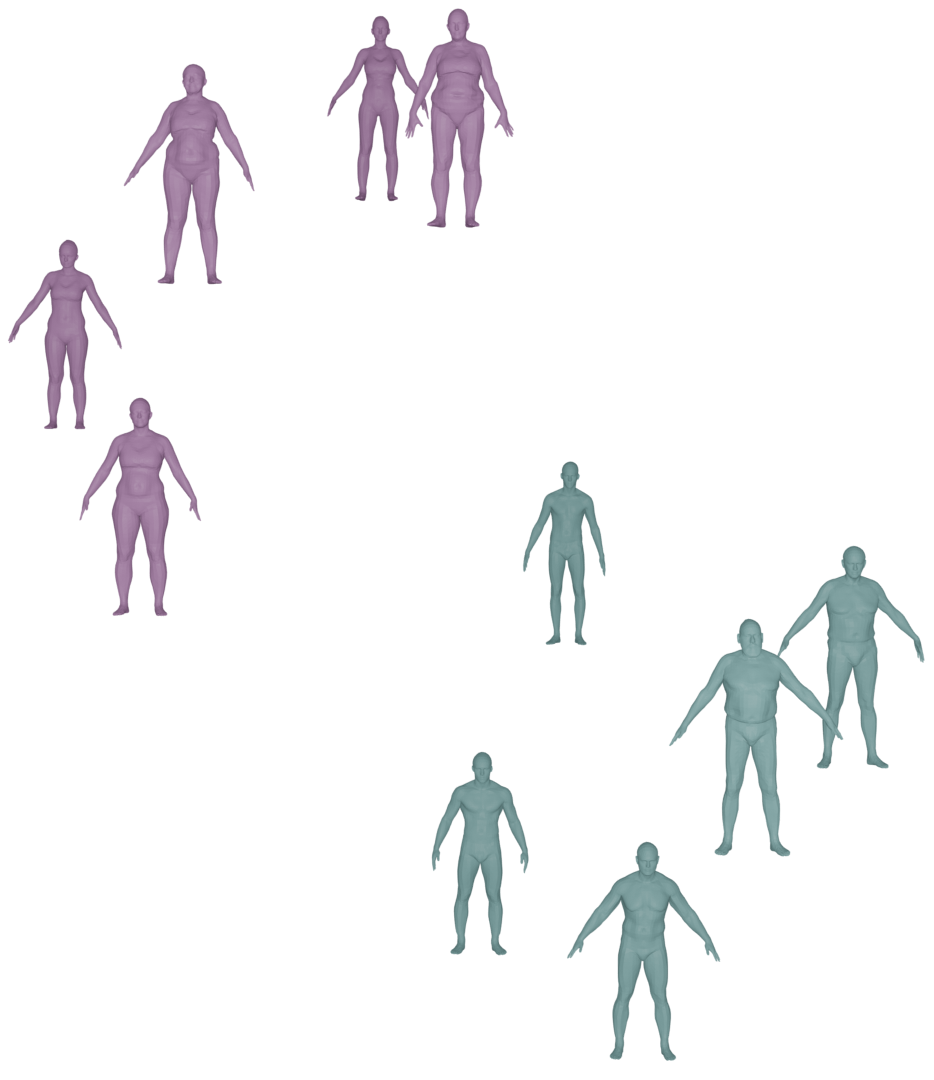}
    \caption{\small Visualizing the distance matrix between ten human body shapes using multidimensional scaling. The geodesic distance naturally clusters the population into male and female shapes.}
    \label{fig:mds}
\end{figure}

\subsubsection{Tangent PCA}
\label{ssec:tangent_space_pca}

\begin{figure}
    \centering
    \includegraphics[width=.48\textwidth]{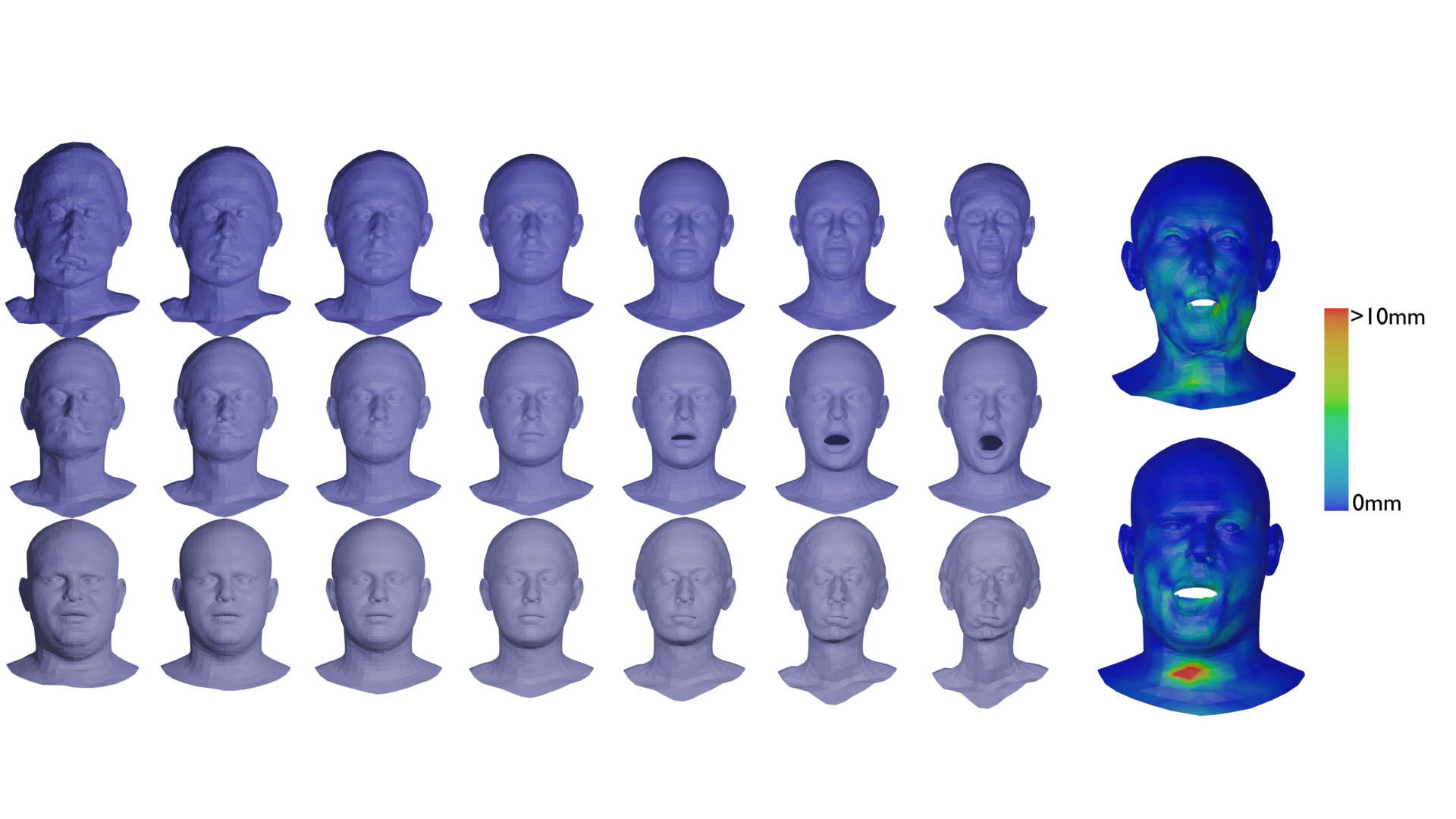}
    \caption{\small Tangent PCA for a set of parametrized surfaces. On the left we display the first three principal component geodesics of a training set. On the right, we display a reconstruction of two elements from a separate testing set, where each vertex is colored based on the Euclidean error of the reconstruction.}
    \label{fig:PCA_CoMA}
\end{figure}

Principal component analysis (PCA) is an important dimensionality reduction technique in statistics for analyzing the variability of data in Euclidean space. More precisely, given data points $x_1,\ldots,x_K \in \R^d$ having zero mean, the goal of PCA is to produce a sequence of linear subspaces $\{W_{\ell}\}_{\ell = 1}^d$ that maximizes the variance of the data when it is projected onto those subspaces~\cite{fletcher2004principal}. This sequence of linear subspaces $W_{\ell} = \operatorname{span}(\{w_1,\ldots,w_{\ell}\})$ for $\ell=1,\ldots,d$ are constructed by finding an orthonormal basis $\{w_1,\ldots,w_d\}$ of $\R^d$ which can be computed as the set of ordered eigenvectors of the sample covariance matrix of the data. Thus, PCA amounts to finding the vectors $\{w_{\ell}\}_{\ell = 1}^d$, which are called the \textit{principal components} of the data.

Extending PCA to manifolds, even in a finite dimensional setting, is not straightforward nor canonical due to the difference in tangent space at each point of the manifold. As a result, several different models and heuristics have been proposed for manifold PCA. Among those, \textit{tangent PCA} \cite{fletcher2004principal} is probably the simplest as it relies on directly linearizing the problem around a single point (the Karcher mean). More specifically, let $(\mathcal{M}, G)$ be a (possibly) infinite-dimensional manifold. Consider data points $x_1,\ldots,x_K \in \mathcal{M}$, and a reference point $\Bar{x} \in \mathcal{M}$, for which a natural choice is e.g. the Karcher mean of the data points. The goal of tangent PCA is to find a set of \textit{principal component geodesics} for the data. By principal component geodesics, we mean a set of geodesics all starting at $\Bar{x}$ whose initial velocities are given by tangent vectors $\{w_{\ell}\}_{\ell = 1}^d \in T_{\Bar{x}}\mathcal{M}$ that are computed as the principal components of the data in the linear space $T_{\Bar{x}}\mathcal{M}$. Thus, tangent PCA amounts to performing standard PCA in $T_{\Bar{x}}\mathcal{M}$, which can be interpreted as finding the ``principal tangent vectors" for the data, i.e., the initial velocities which uniquely determine the geodesics starting at the reference point $\Bar{x}$ along which one has to move on $\mathcal{M}$ in order to maximize the ``variability" of the data.

\begin{figure*}[htbp]
    \centering
    \includegraphics[width=.9\textwidth]{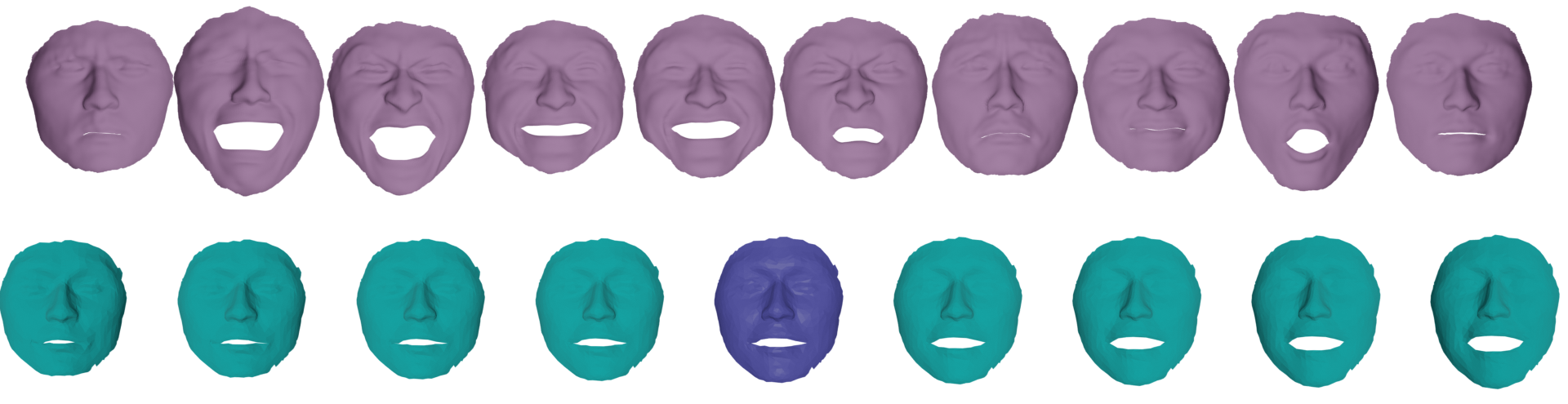}
    \caption{\small First row: a data set of 10 faces with inconsistent mesh structures. Second row: the first principal component geodesic (in the positive and negative directions) from the Karcher mean (purple) of the data set. The principal direction is obtained by tangent PCA.}
    \label{fig:PCA_unparam_faces}
\end{figure*}

We implemented an algorithm for performing tangent PCA when given a set of $K$ surfaces and a reference point, with details given in Algorithm~\ref{alg:tangent_space_pca}. Our method consists of solving $K$ geodesic boundary value problems via Algorithm~\ref{alg:geodesic_bvp_parametrized} (for parametrized surfaces) or Algorithm~\ref{alg:relaxed_match} (for unparametrized surfaces, resp.), using the reference point as the source and each surface in our dataset as respective targets. This produces $K$ geodesics, which we use to estimate $K$ tangent vectors $\{h_k\}$ via finite differences, i.e., by taking the difference between the vertices in the geodesic paths at the first two time points. We then perform PCA on these tangent vectors with respect to the metric $G_{\bar{V}}$ at the reference point. This is specifically done by computing the eigendecomposition $\{\lambda_{\ell}, v_{\ell}\}$ of the $K\times K$ Gram matrix $(G_{\bar{V}}(h_i-\bar{h},h_j-\bar{h}))_{i,j=1,\ldots,K}$, where $\bar{h}=\frac{1}{K}\sum_{k=1}^K h_k$. We then recover the principal component vectors $w_{\ell} = \sum_{k=1}^{K} v_{\ell,k} (h_k-\bar{h})$ and the principal component geodesics by solving initial value problems starting at $\bar{V}$ in the direction of $\lambda_{\ell} w_{\ell}$ using Algorithm~\ref{alg:ivp}. Note that we solve these IVPs in the positive and negative principal directions $\pm \lambda_{\ell} w_{\ell}$ respectively. While we only write the pseudocode for tangent PCA on parametrized surfaces in Algorithm~\ref{alg:tangent_space_pca}, the method works verbatim for the case of unparametrized surfaces, except that the relaxed matching algorithm (Algorithm~\ref{alg:relaxed_match}) is used to solve the $K$ geodesic BVPs. 

To illustrate the effectiveness of tangent PCA, 
we first display the principal component geodesics for an unparametrized dataset of surfaces in Figure~\ref{fig:PCA_unparam_faces}. As a second, more large scale experiment, we analyze the faces of the CoMA dataset \cite{COMA:ECCV18}. As this data comes with known point correspondences, we are able to interpret the data as parametrized surfaces. To evaluate our method we separate the data into a testing set of $\sim2000$ meshes and a training set of $\sim700$ meshes. In Figure~\ref{fig:PCA_CoMA}, we illustrate the principal component geodesics of the training set computed using our method. To reconstruct a target mesh, we then perform an unparametrized geodesic matching from a template to the target with respect to the first 40 tangent PCA basis vectors. In particular, we optimize the relaxed matching energy over all paths where the tangent vectors of the path can be written as a linear combination of the tangent PCA bases. In Figure~\ref{fig:PCA_CoMA}, we also display such a reconstruction of two surfaces from the testing set. When we reconstruct the entire testing set in this way we achieve 75\% of all vertices within a Euclidean error of 1mm. For comparison, the percentage of vertices within 1mm accuracy is 47\% when using traditional PCA and 72\% when using the Mesh Autoencoder methods of \cite{COMA:ECCV18}.

\begin{algorithm}
\caption{Tangent PCA (TPCA)}
\label{alg:tangent_space_pca}
\begin{algorithmic} 

\Procedure{Parametrized\_TPCA}{$V_1,\ldots,V_K$, $\bar V$} \\
\noindent $V_1,\ldots,V_K$ : vertices of parametrized surfaces from sample\\
$\bar V$: vertices of the reference point
\For{$k = 1,\ldots,K$}
    \State $V = \Call{Linear\_Interpolation}{\Bar{V}, V_k}$
    \State $V = \Call{Parametrized\_Geodesic\_BVP}{\Bar{V}, V_{k}, V}$
    \State $h_k = N(V(1/N) - V(0))$ 
\EndFor
\State $\{\lambda_{\ell},w_{\ell}\} = \Call{PCA}{h_1,\ldots, h_K, \bar V}$
\For{${\ell} = 1,\ldots,L$}
    \State $P_{\ell}^{+} = \Call{Geodesic\_IVP}{\Bar{V}, \lambda_{\ell} w_{\ell}}$
    \State $P_{\ell}^{-} = \Call{Geodesic\_IVP}{\Bar{V}, - \lambda_{\ell} w_{\ell}}$
\EndFor
\State \Return $\{P_{\ell}^{+}, P_{\ell}^{-}\}$
\EndProcedure

\Procedure{PCA}{$h_1,\ldots,h_K$,$\Bar{V}$}
\State $\bar{h}=\frac{1}{K}\sum_{k=1}^K h_k$
\State $G_{\Bar{V}}$ = \text{Riemannian $H^2$-metric at $\Bar{V}$}
\State $\Sigma = (G_{\Bar{V}}(h_i-\bar{h},h_j-\bar{h}))_{i,j=1,\ldots,K}$
\State $\{ \lambda_{\ell}, v_{\ell} \} = \text{eigenvalues and eigenvectors of $\Sigma$}$.
\State $w_{\ell} = \sum_{k=1}^K v_{\ell,k} (h_k-\bar{h})$
\State \Return $\{ \lambda_{\ell}, w_{\ell} \}$
\EndProcedure

\end{algorithmic}
\end{algorithm}

\begin{figure*}[htbp]
    \centering
    \includegraphics[width=.9\textwidth]{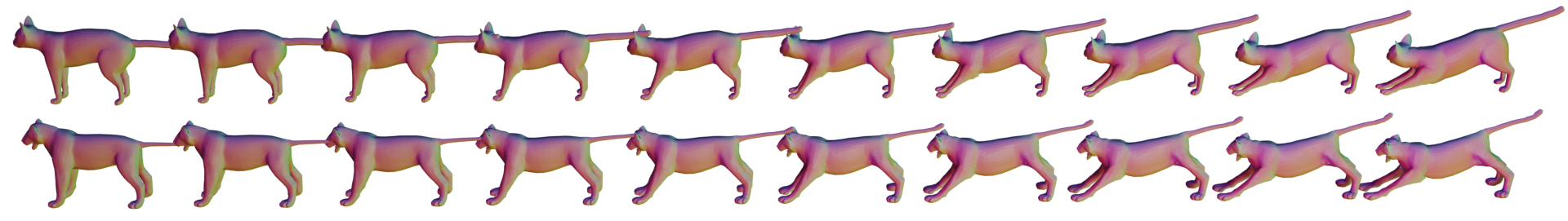}
    \caption{\small  Example of parallel transport using Schild's ladder. We compute the initial tangent vector in the direction of the top geodesic, use Schild's ladder to transport the tangent vector along the geodesic between the leftmost surfaces, and finally compute the geodesic on the the bottom as an IVP. Animations of the obtained motion transfer can be found in the supplementary material and on the github repository.}
    \label{fig:paramMT}
\end{figure*}

\subsection{Parallel transport}
Parallel transport is a method of transporting geometric data (tangent vectors) between different points in a manifold. In our situation this concept has a natural application to motion transfer, as shown in~Figure \ref{fig:paramMT}. Given a geodesic (i.e., a motion) between a source and target surface (e.g. the two cats in Figure \ref{fig:paramMT}), we can transfer the motion to a new source shape (e.g. the lioness in Figure \ref{fig:paramMT}) by parallel transporting the initial velocity from the geodesic motion to the new source shape, and then solving an initial value problem starting at the new source shape with initial velocity given by the parallel transported tangent vector. This procedure requires that we approximate parallel transport of tangent vectors on $\Sol$.  We use an implementation of Schild's ladder to produce a first-order approximation of parallel transport \cite{kheyfets2000schild,guigui2021numerical}. Given a Riemannian manifold $\mathcal{M}$, with $x_0,x_1\in \mathcal{M}$, $h\in T_{x_0} \mathcal{M}$, and letting $V$ be a geodesic such that $V(0)=x_0$ and $V(1)=x_1$, the calculation of parallel transport using Schild's ladder requires one to iteratively compute several small geodesic parallelograms with one side corresponding to a small step along $V$ and the other side being a small step in the direction of $h$. The transport of $h$ for this small step along $V$ is defined to be the log map of the side opposite of $h$. One then repeats the computation of these rungs until reaching $x_1$. An algorithmic explanation of this method is given in Algorithm \ref{alg:parallel_transport} below.
\begin{algorithm}
\caption{Parallel Transport}
\label{alg:parallel_transport}
\begin{algorithmic} 

\Procedure{Parametrized\_Parallel\_Transport}{$V,h,N$} \\
\noindent $V$ : geodesic to transport the tangent vector along\\
\noindent $h$ : tangent vector to be transported\\
\noindent $N$ : number of iterations for Schild's ladder
\For{$i = 1,\ldots,N$}
    \State $W=V(\frac{i-1}{N})+\frac{1}{N}h$
    \State $U=\Call{Linear\_Interpolation}{(V(i/N), W)}$
    \State $M=\Call{Parametrized\_Geodesic\_BVP}{V(i/N),W,U}\left(\frac{1}{2}\right)$
    \State $k=M-V(\frac{i-1}{N})$
    \State $h=\Call{Geodesic\_IVP}{V(\frac{i-1}{N}),2k}(1)-V(i/N)$
\EndFor
\State \Return 
\EndProcedure

\end{algorithmic}
\end{algorithm}

\section{Partial Matching}
\label{sec:partial_matching}
In this final section, we further extend the second-order elastic surface analysis framework introduced in the previous sections by augmenting the surface matching model with the estimation of spatially-varying weights on the source shape. As we will show, this approach will enable us to compare and perform statistics on sets of surfaces which may have incompatible topological properties or exhibit partially missing data. 

\subsection{Limitations of the previous framework}
\label{ssec:limitations}
We start by motivating the need for this extended approach. Indeed, the relaxed matching framework presented so far in~\eqref{eq:relaxed_matching_varifold_symmetric} (as well as its non-symmetric version~\eqref{eq:relaxed_matching_varifold_asymmetric}) is primarily designed for the comparison of complete surfaces with consistent topology, as illustrated by the examples in Figure~\ref{fig:partial_finger_bones}.

Although the matching obtained from~\eqref{eq:relaxed_matching_varifold_symmetric} is inexact and may in practice be able to handle small inconsistencies including topological noise, it remains ill-suited for datasets involving surfaces with significant missing parts or important topological differences (either artifactual or not). Attempting to compare two such surfaces based on model~\eqref{eq:relaxed_matching_varifold_symmetric} is likely to lead to highly singular behaviour in the estimated geodesics and distances. This was already emphasized in the case of planar shapes (such as curves and shape graphs) in the authors' previous publication \cite{sukurdeep2021new}, and can be further observed in the case of 3D surfaces, as seen e.g. in Figure~\ref{fig:spikes} with the formation of geometric artifacts such as the thin arc around the ear of the skull, and in Figure~\ref{fig:partial_finger_bones} with phalanges that shrink to almost zero volume.

To address this shortcoming in our model, we propose to incorporate \textit{partial matching} capabilities in our framework. Extending the idea introduced in \cite{sukurdeep2021new}, we do so indirectly by considering surfaces augmented with a weight function defined on their support, leveraging the flexibility of the varifold representation for that purpose. This will lead to a new matching formulation between pairs of \textit{weighted} surfaces, where, in combination to the geometric matching process, one can vary the weights assigned to different components or parts of the source surface. In particular, this allows us, by setting weights to $0$ in specific areas, to remove parts of the source when they have no corresponding parts in the target surface, as shown in Figure~\ref{fig:partial_finger_bones} and Figure~\ref{fig:partial_spheres}.

\subsection{The varifold norm on the space of weighted surfaces}
\label{ssec:weighted_varifold_norm}
We first define a \textit{parametrized weighted surface} as a couple $(q, \rho)$, where $q \in \mathcal{I}$ is a parametrized surface as previously defined and $\rho: M \to [0,1]$ is a function on the parameter space $M$. For each $(u,v) \in M$, one can interpret $\rho(u,v)$ as the weight assigned to the point $q(u,v)$ on the surface. The primary reason to assume the values of $\rho$ in the interval $[0,1]$ is that we are focusing on the issue of partial matching. In such a scenario, it is indeed natural to impose this constraint, with the interpretation being that the weight function to be estimated in the matching problem introduced below should vanish on parts of the transformed surface that need to get erased to adequately match the target, while remaining roughly equal to $1$ on the other parts. Note that in other situations such as shapes with multiplicities, one could consider more general $\R_+$-valued weight functions. 

Any such weighted surface $(q, \rho)$ can still be represented as the varifold that we write $\mu_{q,\rho} := \rho \cdot \mu_q$, which is defined as the image measure $(q, n_q)_* (\rho\vol_{q})$, where $n_q$ is the unit oriented normal field of $q$, as defined earlier in Section~\ref{ssec:parametrized_surfaces}, and $\rho \vol_{q}$ is the area form on $M$ induced by $q$ rescaled by the weight function $\rho$. 
With this definition, the kernel metrics on varifolds outlined earlier in Section~\ref{ssec:varifolds} immediately induce a fidelity metric between weighted surfaces. 
Specifically, the kernel inner product in $V^*$ between two weighted varifolds $\mu_{q_0,\rho_0}$ and $\mu_{q_1,\rho_1}$ is given explicitly by:
\begin{multline}
\label{eq:norm_var_weighted}
\langle \mu_{q_0,\rho_0},\mu_{q_1,\rho_1}\rangle_{V^*} := \\
{\iint}_{M \times M} \Psi(|q_0 - q_1|)\Phi(n_0 \cdot n_1) \rho_0\rho_1 \vol_{q_0} \vol_{q_1},
\end{multline}
where we have dropped the coordinates $(u_0,v_0)$ and $(u_1,v_1)$ in the above expression for concision. This simply amounts to a weighted version of \eqref{eq:norm_var}. The squared weighted varifold kernel distance $\| \mu_{q_0,\rho_0} - \mu_{q_1,\rho_1} \|_{V^*}^2$ can again be obtained via a quadratic expansion, exactly as in~\eqref{eq:var_discr}. 

\subsection{Relaxed surface matching with weights}
\label{ssec:relaxed_matching_weights}
We are now able to formulate the extension of the matching problem of Section \ref{ssec:relaxed_surface_matching} to weighted surfaces:


\begin{framed}
Given a pair of weighted surfaces $(q_0, \rho_0)$ and $(q_1, \rho_1)$, we consider the variational problem:

\begin{equation}
\label{eq:relaxed_matching_varifold_symmetric_weighted}
\begin{aligned}
& \inf \bigg\{ \int_0^1 G_{q(t)}(\partial_t q(t),\partial_t q(t)) dt ~+\\
& \qquad \lambda_0~\| \mu_{q(0)} - \mu_{q_0} \|_{V^*}^2 + \lambda_1~\| \mu_{q(1), \rho} - \mu_{q_1, \rho_1} \|_{V^*}^2 \bigg\},
\end{aligned}
\end{equation}
where the infimum is taken over paths of immersed surfaces $q(\cdot) \in C^{\infty}([0,1], \I)$, and also over all weight functions $\rho: M \to [0,1]$, with $\lambda_0, \lambda_1 > 0$ being balancing parameters. 
\end{framed}


In this framework, we refer to $(q_0, \rho_0)$ as the \textit{source}, $(q(1), \rho)$ as the \textit{transformed source}, and $(q_1, \rho_1)$ as the \textit{target}. Note that in addition to relaxing the end time constraint, we have also relaxed the initial constraint of $q(0)$ being $q_0$ via a second varifold fidelity term in the model above. Similarly to what was explained in the context of~\eqref{eq:relaxed_matching_varifold_symmetric}, this allows us to choose the topological and mesh properties of the path $q(\cdot)$ independently of those of the source $q_0$, once again paving the way for the use of the efficient multiresolution scheme from~\cite{bauer2021numerical} to numerically solve this matching problem between weighted surfaces. We note that one could also formulate an asymmetric version of problem~\eqref{eq:relaxed_matching_varifold_symmetric_weighted} by instead enforcing the initial constraint $q(0) = q_0$ as in~\eqref{eq:relaxed_matching_varifold_asymmetric}. Furthermore, aside from the constraint of $\rho$ taking its values in $[0,1]$, the variational problem~\eqref{eq:relaxed_matching_varifold_symmetric_weighted} does not involve any cost penalty on the weight function. Yet it would be possible to add regularizers for the weight function to the functional, including for instance the total variation norm of $\rho-\rho_0$ as done in \cite{sukurdeep2021new} for planar shapes so as to promote piecewise constant weight functions. In the context of partial matching, it may also be relevant to enforce $\rho$ to take values close to $0$ or $1$, which can be achieved e.g. by adding a double well pointwise penalty of the form $\int_M (\rho(u,v) (\rho(u,v) - 1))^2 \vol_{q}(u,v)$. A clear downside to including extra regularizers is the added layer of complexity to the matching model due to the presence of extra terms and balancing parameters. For that reason, we decided to focus this work on the above unpenalized formulation.    

\begin{remark}
We emphasize that in~\eqref{eq:relaxed_matching_varifold_symmetric_weighted}, we only allow for weight variations on the transformed source, which lets us model in particular the erasure of parts of $q(1)$ so as to match the target. This is useful e.g. in the context of partial matching problems with missing data in the target shape, see Figure~\ref{fig:partial_finger_bones}. One can easily adapt the model to allow for weight estimation on the target by minimizing over a weight function $\tilde \rho$ defined on $q_1$, with the weights on the transformed source shape being kept fixed. More generally, one could technically model weight variations on both shapes, by jointly optimizing over two weight functions $\rho$ and $\tilde \rho$. However, the latter case requires careful regularization on those functions in order to prevent the trivial solution of setting all weights to 0. We will thus leave the study of this case to future work. 
\end{remark}

\subsection{Numerical optimization with weights}
\label{ssec:numerical_optimization_weights}

\begin{figure*}
    \centering
    \includegraphics[width=.9\textwidth]{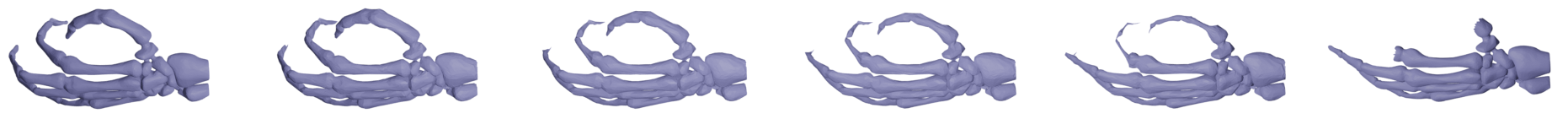}
    \includegraphics[width=.9\textwidth]{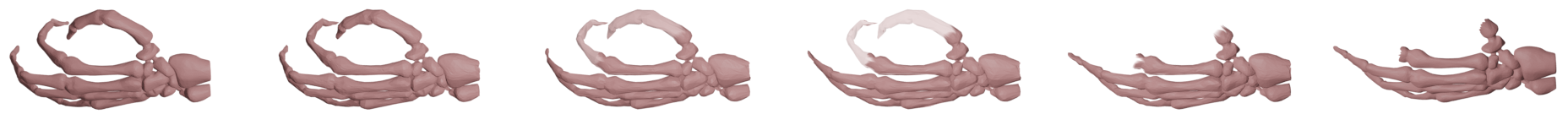}
    \begin{tabular}{m{.55cm}m{2.45cm}m{2.3cm}m{2.25cm}m{2.35cm}m{2.35cm}m{2.4cm}}
    &\phantom{hh}source&\phantom{hh}$q(0)$&\phantom{h}$q(1/3)$&\phantom{h}$q(2/3)$&\phantom{h} $q(1)$&target
    \end{tabular}
    \caption{\small Matching with missing data. We use a complete set of phalanges (i.e., hand bones) as the source, and a different set of phalanges as the target, where some bones on the index finger and thumb were artificially removed. Top row: We matched the surfaces without weight estimation using Algorithm~\ref{alg:relaxed_match}. The parts of the transformed source that are getting matched to the removed bones from the target get shrunk to almost zero volume. The estimated geodesic distance is 117.006. Bottom row: We augment the surfaces with weights and use Algorithm~\ref{alg:relaxed_match_weighted} to match them. Our model correctly ``erases" (i.e., estimates vanishing weights) the appropriate parts of the transformed source to account for the corresponding missing bones on the target. This produces a natural looking geodesic between the source and target, without the production of singularities, with a lower estimated geodesic distance of 114.564.}
    \label{fig:partial_finger_bones}
\end{figure*}

We now discuss our approach for numerically solving the matching problem between weighted surfaces, whose discretization can be performed in similar fashion as previously. A discrete weighted surface $(q, \rho)$ is once again represented as a triangular mesh $(V,E,F)$ as in Section~\ref{sec:numerical_optimization_approach}, while the weight function $\rho$ shall be modelled by its discrete set of values at the center $c_f$ of each face $f \in F$ of the mesh, i.e., by the vector in $[0,1]^{|F|}$ with entries $\rho_f := \rho(c_f)$.

Then, letting $(V, E, F,\rho)$ and $(\Tilde{V}, \Tilde{E}, \Tilde{F},\Tilde{\rho})$ denote the discretizations of two weighted surfaces $(q,\rho)$ and $(\Tilde{q},\Tilde{\rho})$, we can first approximate the varifold inner product: 
\begin{align*}
    &\langle\mu_{q,\rho},\mu_{\Tilde{q},\Tilde{\rho}}\rangle_{V^*} \approx \\
    &\qquad \sum_{f \in F}\sum_{\Tilde{f} \in \Tilde{F}}\Psi(|c_{f}-c_{\Tilde{f}}|)
    \Phi(n_{f}\cdot n_{\Tilde{f}})\rho_f \rho_{\Tilde{f}} \vol_{f}\vol_{\Tilde{f}},
\end{align*}
where $n_f, n_{\Tilde{f}}$ and $\vol_{f}, \vol_{\Tilde{f}}$ are the unit normals and volume forms that have been discretized over the faces $f \in F$ and $\Tilde{f} \in \Tilde{F}$ of the meshes, as outlined in Section~\ref{ssec:metrics_triangular_meshes}. 
The full varifold fidelity term $\| \mu_{q,\rho} - \mu_{\Tilde{q}, \Tilde{\rho}} \|_{V^*}^2$ is then obtained as in~\eqref{eq:var_discr}, via the quadratic expansion of the squared norm.

Equipped with the discretizations of the $H^2$-path energy described in Section~\ref{ssec:discretizing_h2_path_energy}, of the varifold norm described in Section~\ref{ssec:discretizing_varifold_norm}, and of the weighted varifold norm described above, we are led to numerically solve~\eqref{eq:relaxed_matching_varifold_symmetric_weighted} as finite dimensional optimization problem, where the minimization occurs jointly over the vertices of the discretized piece-wise linear path of meshes $V:[0,1] \to \Sol$ and over the discretized weight function $\rho \in [0,1]^{|F|}$. In order to deal with the box constraints on the values of $\rho$, we minimize the discretized matching functional using the bound constrained limited memory BFGS (L-BFGS-B) algorithm~\cite{byrd1995limited}, whose implementation is available through \texttt{scipy}. We summarize the weighted surface matching approach in Algorithm~\ref{alg:relaxed_match_weighted} below. 


\begin{algorithm}
\caption{Relaxed Matching for Weighted Surfaces}
\label{alg:relaxed_match_weighted}
\begin{algorithmic} 
\Procedure{Weighted\_Matching}{$(V_0,\rho_0),(V_1,\rho_1),V, \rho$} \\
\noindent $V_0$ : triangular mesh for the source 
\\$\rho_0$ : weights on the source
\\$V_1$ : triangular mesh for the target
\\$\rho_1$ : weights on the target.
\\$V$ : initial guess for a PL path in $\Sol$. 
\\$\rho$ : initial guess for weights on the transformed source
\State $\operatorname{cost}(V,\rho)=\lambda_0\Call{DistVar}{V(0),V_0} + E(V)$
\State \qquad \qquad \qquad \qquad $+\lambda_1\Call{DistVar}{(V(1),\rho)),(V_1,\rho_1)}$ 
\State $V, \rho = \Call{L-BFGS-B}{V,\rho,\operatorname{cost}}$
\State \Return $V$, $\rho$
\EndProcedure
\end{algorithmic}
\end{algorithm}


\subsection{Partial matching experiments}
\label{ssec:partial_matching_experiments}
To illustrate the capabilities of the weighted surface framework for partial matching, we performed several numerical experiments. In all the figures, we compute the linear interpolation $(1-t)\rho_0 + t \rho$  between the initial and estimated weight function and show this interpolated weight function along the geodesic through a transparency map in order to highlight in a more visual way the effect of weight variations.

First, we demonstrate the benefits of weight estimation when comparing surfaces with missing parts, see the last row of Figure~\ref{fig:main} where we perform a matching with a partially-observed femur bone, and Figure~\ref{fig:partial_finger_bones} where we use an incomplete set of phalanges. Partially observed or incomplete data is a common occurrence in practice and can be due to several factors, including segmentation issues, inconsistent field of views or occlusions during the data acquisition process. Typically, matching surfaces with missing parts using standard elastic surface matching techniques will result in the transformed source getting bent, stretched or compressed in an attempt to fill in some of those missing parts. This can result in unnatural deformations (see the fairly extreme shrinking of phalanges in Figure~\ref{fig:partial_finger_bones}) and in turn in an overestimation of the geodesic distance between these surfaces, making any subsequent statistical shape analysis spurious for datasets of partially observed surfaces. As evidenced by the second row of Figure~\ref{fig:partial_finger_bones}, the proposed approach overcomes this difficulty through the automatic estimation of vanishing weights at the location on the parts of the source shape corresponding to the missing ones of the target.

Second, we computed the matching of surfaces with completely different topologies, as shown with the example of the skulls in Figure~\ref{fig:spikes} and the synthetic examples from Figure~\ref{fig:partial_spheres} and Figure~\ref{fig:partial_torus}. 
The estimation of vanishing weights indirectly allows us to recover several useful types of transformations that are otherwise not achievable in the original model. For instance, it enables the model to erase the thin arc near the left ear of the turquoise skull in Figure~\ref{fig:spikes} as opposed to geometrically shrinking it. It further allows for the splitting of a surface into several connected components, as shown in Figure~\ref{fig:partial_spheres}, as well as the creation of holes when matching surfaces with different genuses, as illustrated by Figure~\ref{fig:partial_torus}. It should be noted, however, that this approach does not directly model topological changes in the mesh of the transformed source (which remains the same along the geodesic), but rather allows us to compare objects with different topologies by erasing parts of the transformed source via the weight function.

\begin{figure}
\centering
    \includegraphics[trim = 20mm 110mm 15mm 110mm ,clip, width=.49\textwidth]{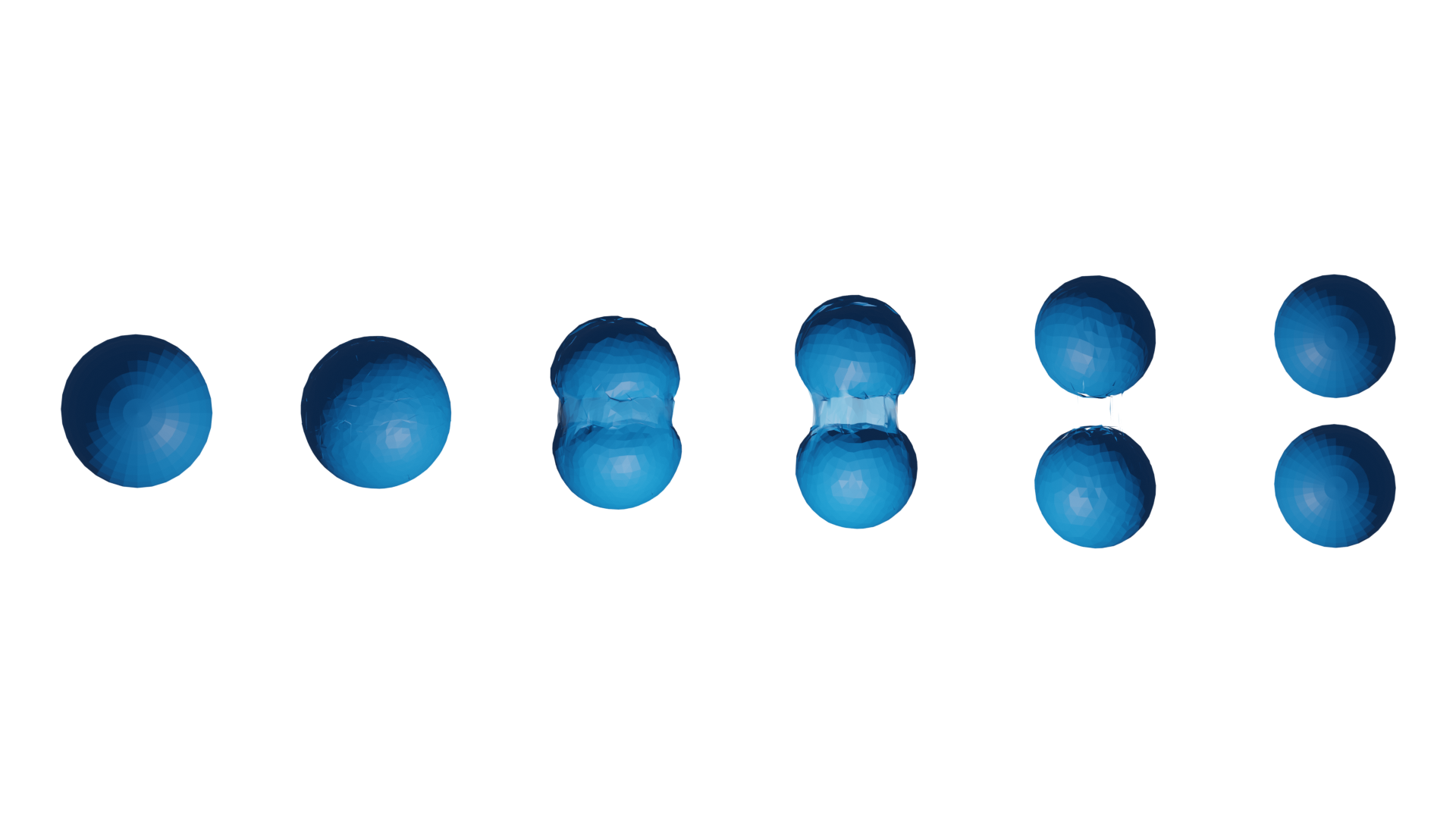}
    \begin{tabular}{m{0.85cm}m{1.1cm}m{1.0cm}m{1.1cm}m{1.1cm}m{1.2cm}}
    source&\phantom{hh}$q(0)$&\phantom{h}$q(1/3)$&\phantom{h}$q(2/3)$&\phantom{h} $q(1)$&target
    \end{tabular}
    \caption{\small Splitting into multiple components. We match a single sphere with two disconnected spheres using Algorithm~\ref{alg:relaxed_match_weighted}. The transformed source $q(1)$ contains a ``bridge" between the two spheres in the target where the algorithm estimates zero weights.}
    \label{fig:partial_spheres}
\end{figure}


\begin{figure}
\centering
    \includegraphics[width=.49\textwidth]{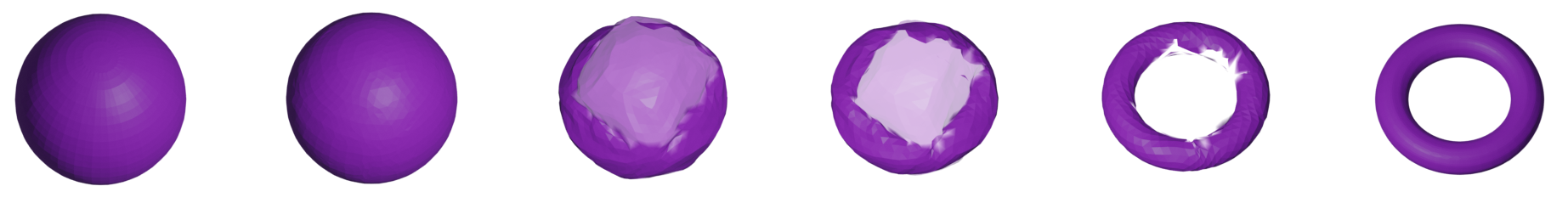}
    \begin{tabular}{m{0.75cm}m{1.1cm}m{1.1cm}m{1.075cm}m{1.1cm}m{1.25cm}}
    source&\phantom{hh}$q(0)$&\phantom{h}$q(1/3)$&\phantom{h}$q(2/3)$&\phantom{h} $q(1)$&target
    \end{tabular}
    \caption{\small Matching with highly inconsistent topological structures. We match a sphere (genus zero surface) and a torus (genus one surface) via Algorithm~\ref{alg:relaxed_match_weighted}. Our model artificially accounts for the creation of a hole, i.e., the change in topology, via the estimation of vanishing weights.}
    \label{fig:partial_torus}
\end{figure}

Lastly, we considered a case of Karcher mean estimation under partial observations. As a proof of concept, we computed the Karcher mean of a set of five distinct hands, each missing a different finger which was artificially removed, see Figure~\ref{fig:kmean_partial}. Following the same principle as the algorithm for Karcher mean estimation presented in Section~\ref{ssec:karcher_mean} (Algorithm~\ref{alg:karcher_mean}), we applied Algorithm~\ref{alg:relaxed_match_weighted} to iteratively solve weighted matching problems from the current Karcher mean estimate to a randomly chosen surface from our dataset. As the initial guess for the Karcher mean, we used a complete hand (i.e. a closed mesh with five fingers) from a different subject. While other choices for the initial guess, e.g. an ellipsoid, are possible, poorly chosen initializations will result in slower convergence to the Karcher mean and  potentially to a lower mesh quality of the estimated Karcher mean. The joint estimation of weights at each successive matching prevents the geometric shrinking of one of the fingers and ultimately results in the realistic looking Karcher mean displayed in Figure~\ref{fig:kmean_partial}. In Figure~\ref{fig:karcher_mean_geodesics}, we also show the computed geodesics from the Karcher mean to each subject. We also report the Riemannian energy of the geodesic path for each of these geodesics in Table~\ref{tab:karcher_mean_distances}. As a point of comparison, we also ran the non-weighted Algorithm~\ref{alg:relaxed_match} between the Karcher mean estimate and each of the corresponding \textit{complete} hands (i.e., without the artificially removed fingers), and report the Riemannian energy of the resulting geodesics in the last column of Table~\ref{tab:karcher_mean_distances}. We observe that the geodesic distance estimates reported in Table~\ref{tab:karcher_mean_distances} are comparable and quite consistent in both scenarios. This point highlights the reliability of the distance estimates obtained with weight estimation and hints at the potential viability of this approach for statistical shape analysis of datasets of partially observed surfaces.
\begin{figure}
    \centering
    \includegraphics[trim = 160mm 10mm 150mm 5mm ,clip,width=.45\textwidth]{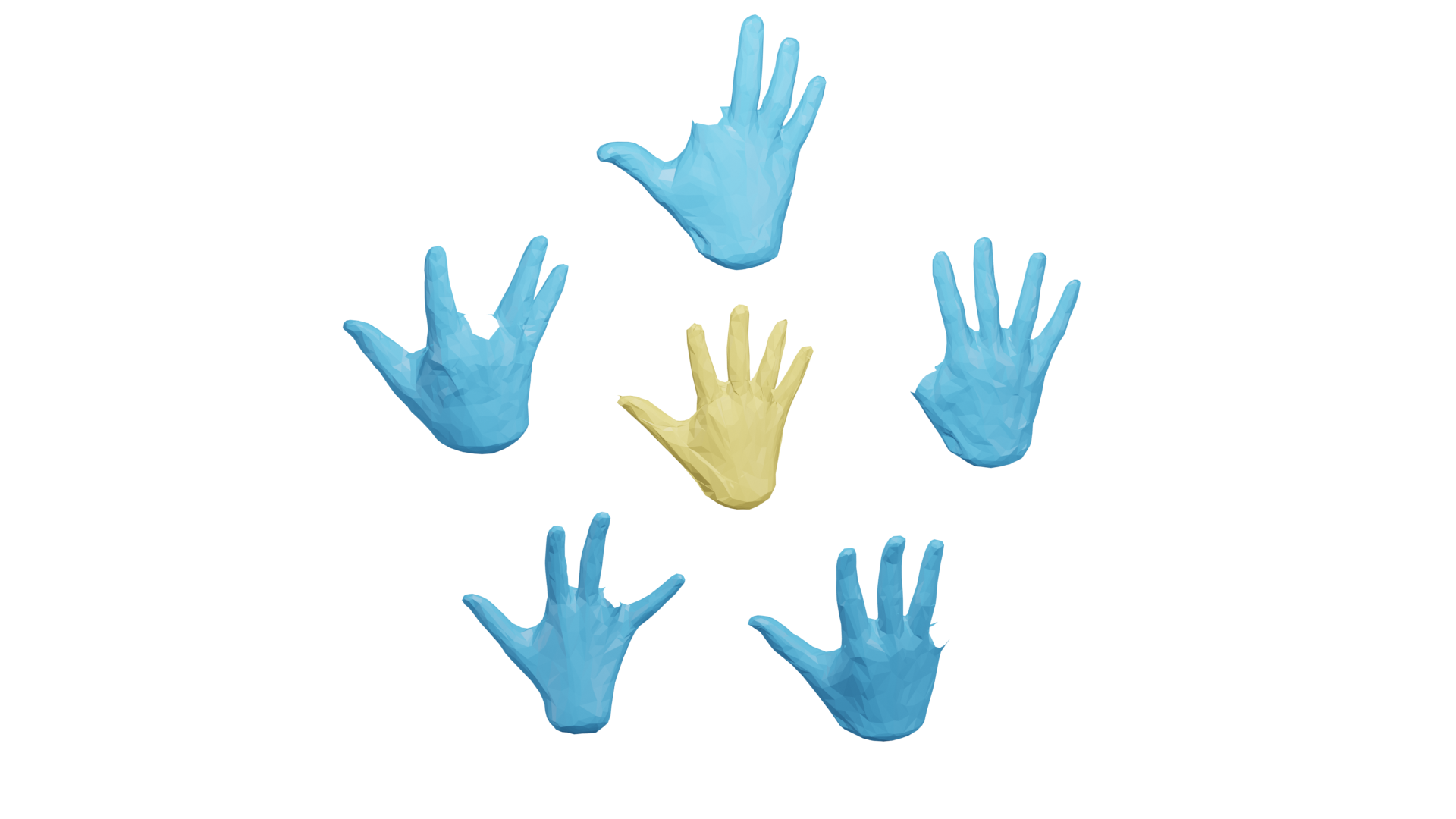}
    \begin{tabular}{m{.75cm}m{2.4cm}m{2.4cm}m{2.4cm}m{2.4cm}m{2.4cm}m{2.4cm}}
    \end{tabular}
    \caption{\small Karcher mean estimation with weights. The data (turquoise) consists of $5$ distinct hands each missing a different finger, and the Karcher mean estimate (yellow) is a complete hand.}
    \label{fig:kmean_partial}
\end{figure}

\begin{figure}
\centering
    \includegraphics[trim = 20mm 0mm 20mm 0mm ,clip, width=.49\textwidth]{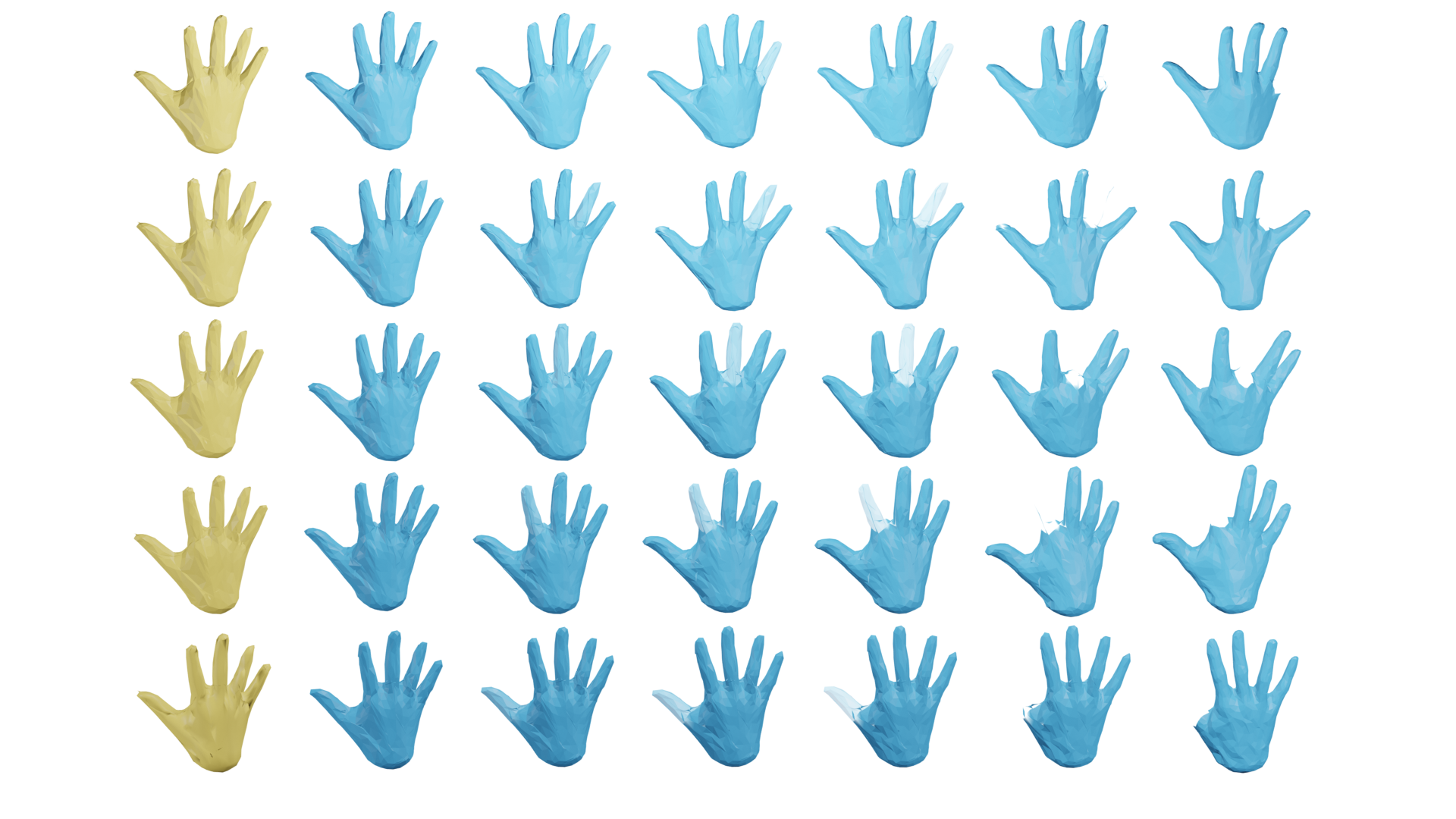}
    \caption{\small Geodesics between Karcher mean estimate (yellow on the left) and data points of the example in Figure~\ref{fig:kmean_partial} (turquoise on the right).}
    \label{fig:karcher_mean_geodesics}
\end{figure}

\begin{table}[h!]
\centering
    \begin{minipage}{.45\textwidth}
        \centering
        
\begin{tabular}{||c c c ||} 
 \hline 
Missing part & Incomplete hand & Complete hand \\
 & (Algorithm~\ref{alg:relaxed_match_weighted}) & (Algorithm~\ref{alg:relaxed_match}) \\ [0.5ex] 
 \hline
Thumb &  0.610 & 0.653 \\ 
Index & 0.687 & 0.708    \\ 
Middle Finger & 0.996 & 1.007\\ 
Ring Finger & 0.708 & 0.789 \\ 
Pinky & 0.642 & 0.799 \\ 
 \hline 
\end{tabular}
\vspace{.1cm}
    \caption{\footnotesize Geodesic distances between the Karcher mean estimate and data points.  \label{tab:karcher_mean_distances}}
        \centering
    \end{minipage}
\end{table}



\section*{Data sources}
\label{sec:data_sources}
The mesh data for our numerical simulations in this paper was obtained from several sources, including the meshes made available by Robert Sumner and Jovan Popovic from the Computer Graphics Group at MIT~\cite{sumner2004deformation}, by Wojtek Zbijewski from the Biomedical Engineering Department at JHU, by Boukhayma et al. from their open source implementation of~\cite{boukhayma20193d}, by the \href{https://www.morphosource.org/?locale=en}{MorphoSource} archive (\url{https://www.morphosource.org}), the TOSCA dataset~\cite{bronstein2008numerical}, the dynamic FAUST dataset~\cite{dfaust:CVPR:2017}, the CoMA dataset~\cite{COMA:ECCV18}, and the dataset of faces from~\cite{vlasic2004multilinear}.

\section{Conclusion and future work}
\label{sec:conclusion}

In this paper, we introduced a mathematical framework and several numerical algorithms for the estimation of geodesics and distances induced by second-order elastic Sobolev metrics on the space of parametrized and unparametrized surfaces. We leveraged our surface matching algorithms to develop a comprehensive collection of routines for the statistical shape analysis of sets of $3$D surfaces, which includes algorithms to compute Karcher means, perform dimensionality reduction via multidimensional scaling and tangent PCA, and estimate parallel transport across surfaces. We also proposed to resolve the issue of partial matching constraints in the situation of missing data and inconsistent topologies through the additional estimation of a weight function defined on the source shape. A more in-depth quantitative evaluation and comparison of our method against other approaches for partial shape matching, such as~\cite{cosmo2016shrec, rodola2017partial} is left for future work.

Finally, we want to mention several limitations of the method presented in this article. First, parameter selection may be an important issue in this framework if no reasonable priors are available for the choice of the $H^2$-metric coefficients or the kernel scale used to compute the varifold relaxation term. As we illustrated in the numerical experiments, those can all have significant influence on the quality of registration and on the behavior of geodesics. While we use different practical strategies such as multiscale schemes to mitigate this issue, a subject of active current investigation is precisely to develop effective approaches to obtain such parameter estimates in a data-driven way, see~\cite{bauer2022elastic} for a recent preprint on this topic in the context of elastic metrics on curves.

A second potential limitation is the choice of initialization for the geodesic path, as the variational problems we tackle are non-convex. In this work, we typically initialize algorithms using a time constant path with either the source or target mesh. However, one could expect extra robustness if more adapted initializations are chosen, which could be computed for instance as the output of some other fast surface matching procedure.

Lastly, while the computational cost of our approach is favorable when compared to other Riemannian frameworks for shape analysis, it still involves running an optimization procedure with quadratic complexity at each iteration. As a result, the numerical pipelines of this paper might become somewhat impractical when working with high resolution meshes (e.g. $O(10^6)$ vertices or more) or for populations with a large number of subjects. One way to overcome this issue would be to leverage deep learning architectures, similar to related work in the case of functional data~\cite{nunez2021srvfnet, chen2021srvfregnet} or planar curves~\cite{hartman2021supervised, nunez2020deep}. This would reduce the computation of quantities such as distances and geodesics to a simple forward pass through a neural network trained from supervised data obtained from our intrinsic $H^2$-metric framework. Conversely, many existing neural network methods applied to surface mesh processing, such as the autoencoder models of~\cite{cosmo2020limp,COMA:ECCV18,huang2021arapreg}, attempt to learn a latent space representation of the surface dataset using various metric priors as regularization. We also plan to investigate in the future the effectiveness of second-order Sobolev metrics to regularize latent space representation learning of mesh autoencoders.

\bibliographystyle{plain}

\end{document}